\pdfoutput=1
\documentclass[lettersize,journal]{IEEEtran}
\usepackage{amsmath,amsfonts}
\usepackage{array}
\usepackage[caption=false,font=normalsize,labelfont=sf,textfont=sf]{subfig}
\usepackage{textcomp}
\usepackage{stfloats}
\usepackage{url}
\usepackage{verbatim}
\usepackage{graphicx}
\usepackage{cite}
\usepackage{graphicx,psfrag,epsf}
\usepackage{enumerate}
\usepackage{bm}
\usepackage{dutchcal}
\usepackage{soul}
\usepackage{color}
\usepackage{setspace}
\usepackage{mathrsfs}
\usepackage{amsmath}
\usepackage{amssymb}
\usepackage{amsthm}
\usepackage{algorithm}
\usepackage{algorithmic}
\usepackage{graphicx}
\usepackage{booktabs}
\usepackage{threeparttable}

\newtheorem{assumption}{Assumption}
\newtheorem{theorem}{Theorem}
\newtheorem{lemma}{Lemma}
\newtheorem{corollary}{Corollary}

\newcommand{\bX}{\mathbf{X}}
\newcommand{\bY}{\mathbf{Y}}

\newcommand{\bB}{\mathbf{B}}

\newcommand{\bbE}{\mathbb{E}}
\newcommand{\bE}{\mathbf{E}}
\newcommand{\bG}{\mathbf{G}}
\newcommand{\bU}{\mathbf{U}}
\newcommand{\bbR}{\mathbb{R}}

\newcommand{\cT}{\mathcal{T}}
\newcommand{\cM}{\mathcal{M}}
\newcommand{\br}{\mathbf{r}}
\newcommand{\cF}{\mathcal{F}}
\newcommand{\bPhi}{\boldsymbol{\Phi}}

\newcommand{\chen}[1]{{\color{green}{\bf{Chen says:}} \emph{#1}}}

\begin{document}


\title{Functional-Edged Network Modeling}
\author{Haijie Xu$^{1}$,  Chen Zhang$^{1}$
\thanks{$^{1}$Tsinghua University, Beijing, China}}

\maketitle

\begin{abstract}
Contrasts with existing works which all consider nodes as functions and use edges to represent the relationships between different functions. We target at network modeling whose edges are functional data and transform the adjacency matrix into a functional adjacency tensor, introducing an additional dimension dedicated to function representation.
Tucker functional decomposition is used for the functional adjacency tensor, and 
to further consider the community between nodes, we regularize the basis matrices to be symmetrical. Furthermore, to deal with irregular observations of the functional edges, we conduct model inference to solve a tensor completion problem. It is optimized by a Riemann conjugate gradient descent method. Besides these, we also derive several theorems to show the desirable properties of the functional edged network model. Finally, we evaluate the efficacy of our proposed model using simulation data and real metro system data from Hong Kong and Singapore.
\end{abstract}



\section{INTRODUCTION}
\label{sec:intro}
Numerous real-world applications produce data that signify interactions between pairs of entities, lending themselves to a natural interpretation as nodes and weighted edges in a network. For instance, in social platforms, individual users can be represented as nodes, with the communication between two users being depicted as edges. In the Internet of Things (IoT), each device can be considered as a node, while the data transmission between two devices can be represented as an edge. 
In urban transportation, each station can be represented as a node, and the passenger flow between stations can be regarded as an edge.

Though many works have been developed for network modeling, they all focus on scalar or vector edges.  In practice, many scenarios involve edges that take on the form of functions within a particular range or domain.  For example, in IoT, signal transmissions between devices are often transformed into the frequency domain, resulting in the formulation of spectral densities. In transportation, the real-time passenger flow between different stations can be continuously observed and formulate a function over time for each day. As illustrated in the left part of Figure \ref{fig:HongKong example}, it shows the passenger flows between four different metro stations in Hong Kong. 

In practice, one approach to address network data with functional edges would be to first sample each function at a grid of points, $t_{1}, \ldots, t_{L}$ and estimate $L$ networks. This could be achieved by separately modeling $L$ networks using existing methods, such as graph embedding \cite{ou2016asymmetric,goyal2018graph,gallagher2023spectral}. For example, matrix factorization-based embedding computes the decomposition of the graph adjacency matrix. To further consider the network dependence between different points, dynamic network models \cite{kazemi2020representation} can be applied as well. 
However, a significant drawback common to these approaches is their requirement that all random functions be sampled at a uniform set of grid points. In practical scenarios, observations of curves often occur at different sets of points, i.e. irregular points, which can hinder the effectiveness of these methods.
Another approach is to use nonparametric smoothing to estimate the functional edges first, and then make slices on regular observation points. However, this smoothing only utilizes data of the same function, potentially overlooking the influence between different edges based on the network structure. This paper takes a different approach by directly treating functional edges as a three-dimensional functional tensor, where the third dimension represents continuous functions defined on a specific domain. We propose a novel functional tensor decomposition for modeling.

\begin{figure*}
  \centering
  \includegraphics[width = 0.9\linewidth]{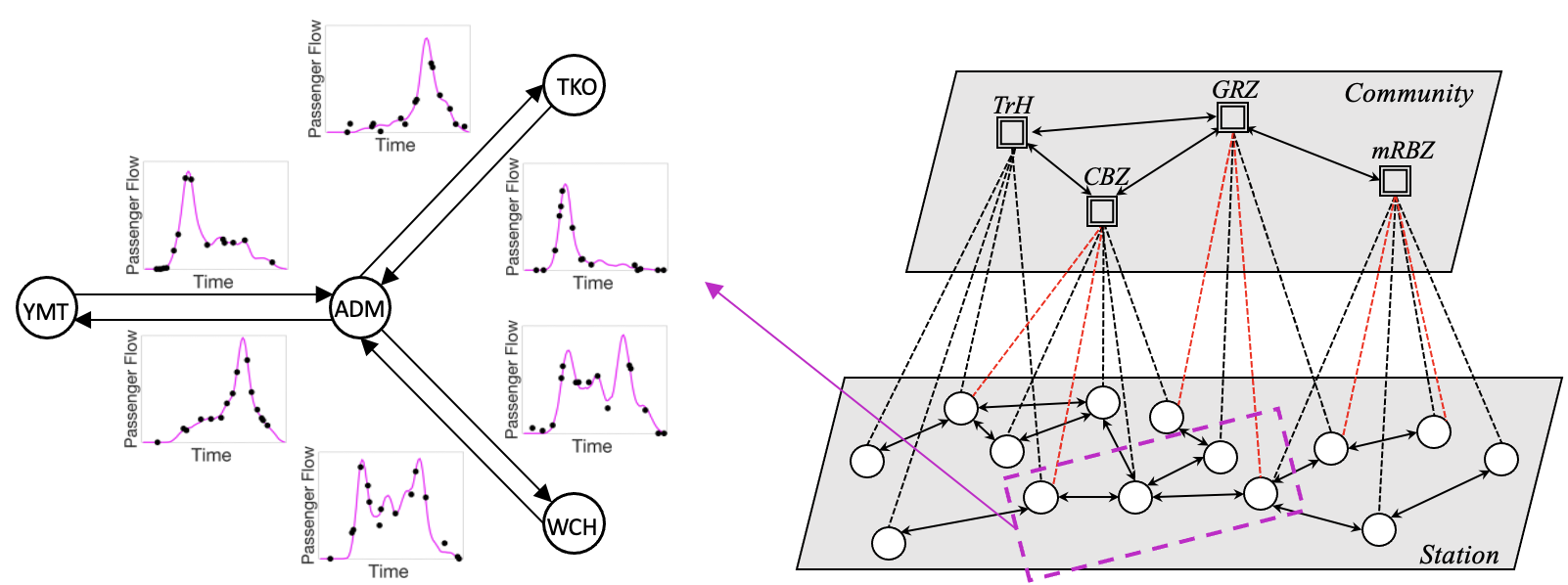}
  \caption{
  Partial passenger flow and community structure of Hong Kong metro system. Left: Passenger flow function (pink line) and its irregular observations (black dots) between four stations shown in the right pink box. Right: Four main communities of the Hong Kong metro system: traffic hub (TrH), central business zone (CBZ),  general residential zone (GRZ) and mixed residential-business zone (mRBZ).
    }
  \label{fig:HongKong example}
\end{figure*}
Furthermore, many large-scale networks exhibit a community structure, characterized by clusters of nodes where edges within each community are dense, while edges connecting different communities are sparse. The community structure of Hong Kong metro system is shown in the right part of  Figure \ref{fig:HongKong example}. Community detection is commonly integrated into network modeling \cite{zhong2015measuring, xu2016network}. Discovering the underlying community structure not only facilitates improving data analysis techniques such as clustering but also allows for a better understanding of the network's overall structure, thereby enhancing the model's interpretability.

Though there are emerging works for functional network modeling, all of them treat node data as functions, and the probabilistic edges describe the dependence structure of nodes. So far to our best knowledge, there is no work dealing with functional-edged network. To fill this research gap, we proposed a so-called Functional Edged Network (FEN) model based on a novel functional Tucker decomposition, which can on the one hand extract features for high-dimensional adjacency functional tensor with irregular observation points, and on the other can describe the community structure of the network by symmetrical decomposition.
To efficiently estimate the FEN model, we propose a Riemann conjugate gradient descent optimization approach. Furthermore, we also discuss the theoretical properties of the FEN model.

The remainder of this paper is organized as follows. Section \ref{sec:literature} reviews the literature on related topics to our problem. Section \ref{sec:model} describes our specific problem formulation and gives its theoretical properties. Section \ref{sec:optimization} introduces our model estimation algorithm in detail. 
Section \ref{sec:simulation} presents thorough numerical studies based on simulated data. Section \ref{sec:case} applies our proposed FEN to the real-world case study to further illustrate its efficacy and applicability. Some conclusive remarks are given in Section \ref{sec:conlusion}.

\section{LITERATURE REVIEW}
\label{sec:literature}

\subsection{High Dimensional Functional Data Modeling}
If we neglect the spatial structure of the network, we can regard the $m(m-1)$ functional edges as multivariate functions. In such cases, modeling methods based on principal component analysis (PCA) are commonly employed.

In particular, for functional data with regular observation points, traditional PCA can be revised and applied for modeling. Specifically, \cite{VPCA} proposed a vectorized PCA (VPCA) algorithm by combining multiple functions into one long function data, and using classical PCA for decomposition. However, this brute vectorization neglects the correlation between different functions and hence loses a lot of estimation accuracy. An alternative method is presented by \cite{MFPCA}, known as multivariate functional PCA (MFPCA). This approach regarded these multiple functions as the different samples of one function and combined them as a function matrix to do PCA. Following it, \cite{FPCA_ZC} further assumed its functional scores are sparse, and could better deal with high-dimensional functions with diverse features. Besides these, there are also many other functional PCA (FPCA) methods that can describe specific data features. For instance, \cite{foutz2010research} smoothed the functional data before applying the FPCA to get the smooth functional shape. \cite{bali2011robust} adapted the projection-pursuit approach to the functional data setting to achieve the robust estimator for functional PCA. However, all of the above methods cannot be directly applied to irregularly observed functional data, unless some data alignment algorithms are preprocessed. Furthermore, these methods do not address the smoothness property of functional data as well. 
 
 To deal with irregularly observed functional data, \cite{SIFPCA} proposed a sparse irregular FPCA (SIFPCA)  based on kernel smoothing algorithm to estimate the covariance function with sparsely observed functions. Then conditional expectation algorithm is used to estimate the FPCA score. Later, \cite{SIFPCAImproved} further extended it to more strict settings, where observation points are not required to be fixed but can be treated as random variables with a given distribution. Besides these methods, the mixed effect model is also used to estimate the sparse FPCA \cite{rice2001nonparametric}, but there is a problem that the high-dimensional covariance matrix may be numerically unstable \cite{FPCAReview}. To solve this problem, \cite{james2000principal} and \cite{james2003clustering} proposed the reduced rank model that avoided the above potential problems of the mixed effects model. Then \cite{zhou2008joint} used penalized splines to model the principal components of FPCA, cast the model into the mixed effect model and extended it for two-dimensional sparse irregular cases.

However,  all these methods can neither address network community structure nor be directly used for functional-edged network modeling.

\subsection{Functional Network Analysis}
Functional graphical models have garnered increasing research attention in recent years, as evidenced by works such as \cite{DynNet_MLEBasic, DynNet_DoubleFun, DynNet_BayesEdge}. These studies share a common approach, considering the node data as functions, while the edges are conceptualized as "probabilistic edges", symbolizing the relationships between different functions. The primary objective is to achieve accurate estimation of these edges.

Among different methods, the traditional Gaussian graphical model  \cite{chun2013joint,danaher2014joint} is the most popular one. Initially,\cite{DynNet_MLEBasic} treats an edge as a partial cross-correlation function between nodes. It first conducted FPCA on the data, and then calculated the precision matrix of PC scores of different nodes, as the adjacency matrix. 
Considering the edges are sparse, it assumes the probabilistic edge is fixed. Building upon this foundation, \cite{DynNet_DoubleFun} took a further step by considering that the probabilistic edges can also be functions, introducing the concept of the doubly functional graphical model. \cite{tsai2023latent} also considered functional data from multiple modes to be multimodal. It decomposes the multimodal data into different spaces and assumes the decomposition coefficients formulate a Gaussian graphical model. These above models assume Gaussian distribution on functional data. Besides, several nonparametric models are also proposed. \cite{li2018nonparametric} proposed a nonparametric graphical model based on the additive conditional dependence, while \cite{solea2022nonparametric} introduced an additive function-on-function regression model to describe the nonparametric partial correlation among different nodes. 

If all the functional edges are regularly observed at equally spaced grids, we can treat them as dynamic networks, for which many models have been proposed such as stochastic block model \cite{matias2017statistical}, exponential random graph model \cite{lee2020model}, latent factor model \cite{MTR}, latent feature relational model \cite{heaukulani2013dynamic} and latent space model (LSM) \cite{sewell2017latent}. Here we would like to address LSM, originally proposed by \cite{hoff2002latent}.
The core concept behind LSM lies in assigning each actor a vector in a low-dimensional latent space. The pairwise distances between these vectors, calculated using a specified similarity measure, determine the probabilities of connections within the network. LSM interprets these latent features as unmeasured characteristics of nodes, implying that nodes closer together in the latent space are more likely to establish connections. 
This interpretation naturally explains the presence of high levels of homophily and transitivity observed in real-world networks. Additionally, LSM can effectively capture and describe the community structure within the network by regarding each latent state as a community and estimating the likelihood of nodes belonging to each community.

In particular, \cite{sewell2017latent} introduced a community detection method within dynamic network data by LSM, by assuming the latent position of each node follows a Gaussian mixture model whose distribution can dynamically change over time. Recently, \cite{loyal2023eigenmodel} further proposed an LSM for the multilayer dynamic network, where a symmetric tri-matrix decomposition is developed to capture different community structures in different layers. 
Similarly, \cite{robinson2012detecting} proposed a method of discovering change points in network behavior via using a k-dimensional simplex latent space. The above methods build community structure well through LSM which can be considered in our model.

As previously mentioned, all the methods discussed above are graphical models, wherein nodes contain data, and edges represent "probabilistic" conditional dependencies between nodes.
In summary, although the goals of existing works are different from our paper, their modeling methods shed light upon our model. 


\subsection{Tensor Analysis For Networks}
For cases when all the functional edges are regularly obs at equally spaced grids, another way is to stack the adjacency matrices of all the sampling points as an adjacency tensor.  Subsequently, tensor decomposition algorithms, such as CANDECOMP/PARAFAC (CP) decomposition, tucker decomposition, etc, can be employed for modeling \cite{kazemi2020representation}. For extrapolation, 
\cite{papalexakis2012parcube} proposed a CP decomposition that introduced a non-negativity constraint to enhance interpretability. The decomposition results were used to find the most active nodes or time points. \cite{xiong2010temporal} proposed a probabilistic CP factorization. It also imposed a smoothness prior to the temporal vectors corresponding to using time as a regularizer. Additionally, \cite{yu2017link} presented another way of incorporating temporal dependencies into the embeddings with decomposition methods. It decomposed the adjacency matrices into both time-invariant and time-variant components and regularized the time-variant components to be projected onto feature space that ensures neighboring nodes have similar feature vectors.

In addition to the mentioned methods, several genetic tensor decomposition and completion techniques consider the incorporation of smoothing constraints. In particular, \cite{2DContinueDecomposition} proposed a decomposition algorithm for matrix decomposition, the two-dimensional tensor,  where one dimension represents continuous time. Similar to CP decomposition in tensors, this paper used two fixed basis sets, one is discrete and the other is continuous. Additionally,  \cite{CPDecSmooth} considered ``smoothness'' constraints as well as low-rank approximations and proposed an efficient algorithm for performing tensor completion that is particularly powerful regarding visual data.
Furthermore, \cite{han2023guaranteed} used the Reproducing Kernel Hilbert Space (RKHS) theory to deal with the smoothness of tensor. It proposed a functional CP decomposition method, where the basis functions of the continuous dimension belong to the RKHS. However, it did not consider the completion problem under irregular observations.
Meanwhile, \cite{wang2014low} proposed a novel spatially-temporally consistent tensor completion method for recovering the video missing data. This approach introduced a new smoothness regularization along the temporal dimension to leverage temporal information between consecutive video frames.
\cite{zhou2023partially} considered completing partially observed tensors in the presence of covariates through a regression approach. It performed a CP decomposition on the regression coefficients and added sparse and smooth constraints to the decomposed vectors to do the completion.
\cite{chen2021bayesian} integrated low-rank tensor factorization and vector autoregressive (VAR) process into a single probabilistic graphical model, to characterize both global and local consistencies in large-scale time series data. \cite{mcneil2021temporal} addressed graphs with nodes representing time series data and edges signifying static connection strengths between nodes.  It learned a low-rank, joint decomposition of the data via a combination of graph and time dictionaries. 

In summary, though the above methods cannot be directly used for modeling functional-edged networks with irregularly observed points,
the above methods motivate us to use tensor analysis to describe the functional network, which is also the innovation of our research.

\section{FEN MODEL}
\label{sec:model}
In this section, we present our FEN model in detail. Firstly, we introduce the notation system of this paper. Then we present the model formulation using Tucker decomposition which considers community structure. Next, we address irregular observation points by employing a mask tensor and formalizing smoothing constraints in the model inference framework to better capture the smoothness of functions. 

\subsection{Notation}
We use a lowercase letter $a$ to denote a scalar, a boldfaced lowercase letter $\mathbf{a}$ to denote a vector, a boldfaced uppercase letter $\mathbf{A}$ to denote a regular matrix or tensor, and a stylized uppercase letter $\mathcal{A}$ to denote a matrix or tensor that contains continuous dimensions.
For a three-dimensional tensor $\mathbf{A} \in \bbR^{d_1 \times d_2 \times d_3}$, we use $A_{i_1,i_2,i_3}$ to denote its $(i_1,i_2,i_3)$ element. If some dimensions of tensor $\mathcal{A}$ are continuous, we use parenthesized subscripts to indicate this. For example, $A_{i_1,i_2,(t)}$ denotes the $(i_1,i_2,(t))$ element of tensor $\mathcal{A}$ where the third dimension is continuous functional data in the domain $t\in \cT$. We use  $\mathcal{A}_{[i_1=j]}$ to denote a slice of tensor $\mathcal{A}$, where only fixed dimensions are in square brackets.   
For convenience, we can use the abbreviated notation $A_{i_1i_2i_3}$ instead of $A_{i_1,i_2,i_3}$ without ambiguity.
Then we can define the matricization of tensor, e.g. $\mathbf{A}_{(1)} \in \bbR^{d_1 \times d_2d_3}$ as follows,
\begin{equation*}
    (\mathbf{A}_{(1)})_{i_1,(i_2-1)d_3+ i_3}= A_{i_1i_2i_3},
\end{equation*}
where $i_1 = 1,\cdots,d_1, i_2 = 1,\cdots, d_2, i_3 = 1,\cdots, d_3$. Then $\mathbf{A}_{(2)}$, $\mathbf{A}_{(3)}$ and so on can be defined similarly.

The inner product of $\mathbf{A},\mathbf{B} \in \bbR^{d_1 \times d_2 \times d_3}$ is $\langle\mathbf{A}, \mathbf{B}\rangle = \sum_{i_1 = 1}^{d_1}\sum_{i_2 = 1}^{d_2}\sum_{i_3 = 1}^{d_3}A_{i_1i_2i_3}B_{i_1i_2i_3}$. The Frobenius norm (F norm) of tensor $\mathbf{A} \in \bbR^{d_1 \times d_2 \times d_3}$ can be defined as $||\mathbf{A}||_{F} = \sqrt{\langle \mathbf{A},\mathbf{A}\rangle} = \sqrt{\sum_{i_1 = 1}^{d_1}\sum_{i_2 = 1}^{d_2}\sum_{i_3 = 1}^{d_3}A_{i_1i_2i_3}^2}$. Similarly, if $\mathcal{A} \in \bbR^{d_1 \times d_2 \times \cT}$, its F norm can be defined as $||\mathcal{A} ||_F = \sqrt{\sum_{i_1 = 1}^{d_1} \sum_{i_2 = 1}^{d_2} \int_{\cT} A_{i_1i_2(t)}^2dt}$

At last, we need to define the marginal product, $\times_1$: $\bbR^{r_1\times r_2 \times r_3} \times \bbR^{d_1 \times r_1} \longmapsto \bbR^{d_1 \times r_2 \times r_3}$ as follows,
\begin{equation*}
    (\mathbf{A} \times_1 \mathbf{C})_{i_1j_2j_3} = \sum_{j_1 = 1}^{r_1} A_{j_1j_2j_3}C_{i_1j_1},
\end{equation*}
where $\mathbf{A}\in \bbR^{r_1\times r_2 \times r_3}, C \in \bbR^{d_1 \times r_1}$ and $i_1=1,\cdots,d_1\quad j_2 = 1,\cdots,r_2 \quad  j_3 = 1,\cdots, r_3$. Then $\times_2$, $\times_3$ and so on are defined similarly. For continuous tensor $\mathcal{A}$ we can also similarly define the marginal product by using integration instead of summation as follows,
\begin{equation*}
    (\mathcal{A}\times_1 \mathcal{C})_{i_1j_2j_3} = \int_{\cT}A_{(t)j_2j_3}C_{i_1(t)}dt
\end{equation*}
where $\mathcal{A}\in \bbR^{\cT\times r_2 \times r_3}, \mathcal{C} \in \bbR^{d_1 \times \cT}$ and $i_1=1,\cdots,d_1\quad j_2 = 1,\cdots,r_2 \quad  j_3 = 1,\cdots, r_3$.

\subsection{Model Formulation}
We assume the functional network has $m$ nodes and at most $m(m-1)$ edges. The weight value of each edge can change over time. Specifically, we assume that the edge between node $i$ and $j$ has a weight function denoted by $X_{ij(t)}$, which is a first-order continuously differentiable function at a closed interval $\cT=[T_s,T_e]$ with respect to $t$.

We can represent the time-varying edge weights using an adjacency functional tensor $\mathcal{X} \in \bbR^{m \times m \times \cT}$, which is a three-dimensional functional tensor consisting of $m(m-1)$ weight functions as described above. The third dimension of $\mathcal{X}$ is continuous and represents the weight functions in $\cT = [T_s, T_e]$. 

To describe the community structure in the network, we assume there are in total $s$ communities. Define $\bPhi\in \bbR^{m\times s}$ where $\Phi_{ia}$ represents the possibility of the community $a$ containing node $i$. Suppose the strength of the connection from community $a$ to community $b$ at point $t$ as $C_{ab(t)}$. Then we conduct the following decomposition on $X_{ij(t)}$
\begin{equation}
\begin{aligned}
    X_{ij(t)} & =  \sum_{a = 1}^s\sum_{b=1}^s\Phi_{ia} C _{ab(t)} \Phi_{jb} \\
    & s.t. \quad \bPhi^{T} \bPhi = \mathbf{I}_{s}.
    \label{eq:community structure}
    \end{aligned}
\end{equation}
Here $\mathbf{I}_s \in \bbR^{s \times s}$ denotes the identity matrix of order $s$. On the one hand, the orthogonality can avoid an unidentifiable model. For example, if we let $\Phi_{ia}^* = \sqrt{\lambda}\Phi_{ia}, \Phi_{jb}^* = \sqrt{\lambda}\Phi_{jb}$ and $C_{ab(t)}^* = \frac{C_{ab(t)}}{\lambda}$, then $\Phi_{ia}^* C_{ab(t)}^*\Phi_{jb}^* = \Phi_{ia} C_{ab(t)}\Phi_{jb}$. As such, to make the model identifiable, we set $\sum_{i=1}^m \Phi_{ia}^2 = 1, a = 1,\cdots, s$. 
On the other hand, we have $0 \leq \Phi_{ia}^2 \leq 1$ which can make $\Phi_{ia}^2$ better interpretable, i.e., as the possibility of the community $a$ contains node $i$. The higher the value of $|\Phi_{ia}|$, the more likely node $i$ is in the community $a$. Additionally, the sign of $\Phi_{ia}$ denotes the attitude of node $i$ to community $a$. For example, in social networks, people can express likes and dislikes about the group they belong to. 
Last, $\sum_{i = 1}^{m}\Phi_{ia}\Phi_{ib} = 0, a \neq b$ can further guarantee the diversity of different communities.



To further describe the temporal features of $C_{ab(t)}$, we assume $C_{ab(t)}$ as a linear combination of $K$ functional bases $\{\mathcal{g}_{k}(t),k =1,\ldots,K\}$ as: 
\begin{equation*}
\begin{aligned}
  C_{ab(t)} =\sum_{k=1}^{K} B_{abk} \mathcal{g}_{k}(t).
    \label{eq:dec of C}
\end{aligned}
\end{equation*}
Combine Equation (\ref{eq:community structure}), we can get 
\begin{equation*}
\begin{aligned}
    X_{ij(t)}=\sum_{a = 1}^{s}\sum_{b = 1}^s\sum_{k = 1}^{K} B_{abk}\Phi_{ia}\Phi_{jb}\mathcal{g}_{k}(t).
    \label{eq:observe single point}
\end{aligned}
\end{equation*}
The above equation can be rewritten in the tensor form as:
\begin{equation}
\label{eq:tucker for X}
    \mathcal{X} = \bB \times_{1} \bPhi \times_{2} \bPhi \times_{3} \mathcal{G}.
\end{equation}
Here $\bB \in \bbR^{s \times s \times K}$, $\bPhi \in \bbR^{m \times s}$, $\mathcal{G} \in \bbR^{\cT \times K}$, where $  \mathcal{G}_{[i_2 = k]} \triangleq \mathcal{g}_k(\cdot)$. It is to be noted that without loss of generality, for convenience, we assume the $\mathcal{X}$ has been centralized.

 
Considering the data may include noise, we further define the observation functions as follows,
\begin{equation}
\begin{aligned}
\label{eq:continue Y=X+E}
    \mathcal{Y} = \mathcal{X} + \mathcal{E} 
     = \bB \times_{1} \bPhi \times_{2} \bPhi \times_{3} \mathcal{G} + \mathcal{E},
\end{aligned}
\end{equation}
where $\mathcal{Y} \in \bbR^{m \times m \times \cT}$ denotes the observation function and $\mathcal{E} \in \bbR^{m \times m \times \cT}$ denotes the noise function. Each point in $\mathcal{E}$, i.e., $E_{ij(t)},i=1,\cdots,m, \quad j = 1,\cdots,m, \quad t \in \cT$  follows a normal distribution with mean of $0$ and variance of $\sigma^2$ and is independent with each other point. 

Equation (\ref{eq:continue Y=X+E}) expresses the observation function when we only have one sample. In reality, if we have multiple samples, we need to introduce a fourth dimension and the data become $\mathcal{X},\mathcal{Y},\mathcal{E} \in \bbR^{m \times m \times \cT \times N}$ where $N$ represents the number of samples. Since the fourth dimension does not need decomposition, we have
\begin{equation}
\begin{aligned}
\label{eq:continue Y=X+E dim=4}
    \mathcal{Y}= \mathcal{X} + \mathcal{E}  = \bB \times_{1} \bPhi \times_{2} \bPhi \times_{3} \mathcal{G} \times_{4}\mathbf{I}_N+ \mathcal{E},
\end{aligned}
\end{equation}
where $\mathbf{I}_N\in \bbR^{N\times N}$ is an identity matrix of order $N$.

\subsection{Model Inference}
Now we introduce how to conduct parameter estimation of $\{\mathcal{X},\bB,\bPhi,\mathcal{G}\}$  for FEN. Generally, we need the functional tensor $\mathcal{X}$ to be smooth in the third dimension to represent a continuum of values across a domain. This can be achieved by adding a $l_2$ loss on the derivative of the functional dimension, as the smoothing constraints when we do the estimation. Because of the Tucker decomposition, we can add this smoothing constraint to the corresponding decomposed bases, i.e., $\mathcal{g}_{k}(t), k=1,\ldots,K$, to guarantee the smoothness of $\hat{\mathcal{X}}$. Then we can rewrite our FEN model as solving the following optimization problem, 

\begin{equation}
    \begin{aligned}
         \{\hat{\mathcal{X}},\hat{\bB},\hat{\bPhi},\hat{\mathcal{G}}\} = & \mathop{\arg \min}\limits_{\mathcal{X}, \bB, \Phi, \mathcal{G}} \frac{1}{2}||\mathcal{Y} - \mathcal{X}||_{F}^2 + 
         \sum_{k = 1}^{K} \alpha_k\int_{\cT}({\mathcal{g}_k}^{'}(t))^2\,dt\\
         s.t. &\mathcal{X} = \bB \times_{1} \bPhi \times_2 \bPhi \times_3 \mathcal{G}\\
         &\bPhi^T \bPhi = \mathbf{I}_{s}\\
         &{\mathcal{G}}^T \mathcal{G} = \mathbf{I}_{K} \\
         &\mathcal{g}_{k} \in \mathbf{C}^1(\cT),
    \end{aligned}
    \label{eq:FEN continue}
\end{equation}
where $\alpha_k$ is the smoothing constraint coefficients, $\mathbf{I}_K \in \bbR^{K \times K}$ is an identity matrix of order $K$. $\mathbf{C}^1(\cT)$ denotes the set of all first-order continuous differentiable functions on $\cT$ . 

The optimization problem presented above gives theoretical properties in continuous space. In reality, each edge $Y_{ij(t)n}$ can only be observed at a set of points. Consider the observation points of different edges are various. 
We find the smallest observation resolution of all edges and define it as the global observation resolution, with a corresponding set of regularly spaced observation points as $\Tilde{\cT} = \{t_{l}=T_s + \frac{l}{L} (T_e - T_s) | l = 1,2,\dots, L\}$. It is to be noted that $L$ is expected to be big enough to ensure that any observation point of any function sample is a subset of $\Tilde{\cT}$.
Then the $l$th observation of edge between node $i$ and node $j$ of sample $n$ can be written as follows,
\begin{equation*}
\begin{aligned}
\label{eq:observe single point2}
    Y_{ij(t_l)n} &= X_{ij(t_l)n} +E_{ij(t_l)n}\\
    &=\sum_{a = 1}^{s}\sum_{b = 1}^s\sum_{k = 1}^{K} B_{abkn}\Phi_{ia}\Phi_{jb}\mathcal{g}_{k}(t_{l}) + E_{ij(t_l)n}, \forall t_l \in \Tilde{\cT},
\end{aligned}
\end{equation*}
where $E_{ij(t_l)n}$ is the noise following an independent and identical normal distribution with a mean of $0$ and variance of $\sigma^2$.

If we can observe edges at all the points of $\Tilde{\cT}$, We can get the fully observed discrete tensor as follows,
\begin{equation*}
\begin{aligned}
     \bY^{F} = \bX + \bE
     = \bB \times_{1} \bPhi \times_{2} \bPhi \times_{3} \bG + \mathbf{E},
\end{aligned}
\end{equation*}
where $\bY^{F} \in \bbR^{m \times m \times L \times N}, \bX \in \bbR^{m \times m \times L \times N}, \bE \in \bbR^{m \times m \times L \times N}, \bG \in \bbR^{L \times K}$ denote discretized versions of $\mathcal{Y}, \mathcal{X}, \mathcal{E}$ and $ \mathcal{G}$ at the corresponding points in $\Tilde{\cT}$ respectively.


In reality, because each edge is only observed at a subset of $\Tilde{\cT}$, which is different for different edges, we express these irregular observations using a mask tensor $\boldsymbol{\Omega} \in \bbR^{m \times m \times L \times N}$, whose value is $1$ at the points that $\bY^{F}$ are observed and $0$ otherwise. Then we have
\begin{equation*}
    \bY = \boldsymbol{\Omega} * \bY^F,
\end{equation*}
where $*$ denotes the element-wise product,  $\bY$ is the structured observation tensor in reality. With it, we can observe $\{\mathcal{X},\bB, \Phi, \mathcal{G}\}$ by rewriting Equation (\ref{eq:FEN continue}) as follows:
\begin{equation}
    \begin{aligned}
         \{\hat{\bX},\hat{\bB},\hat{\bPhi},\hat{\bG}\} = & \mathop{\arg\min}\limits_{\bX, \bB, \Phi, G} \frac{1}{2}||\mathcal{P}_{\boldsymbol{\Omega}}(\bY - \bX)||_{F}^2 + 
         \frac{1}{2}\sum_{k = 1}^K\alpha_k \mathbf{g}_k^{T} \mathbf{H} \mathbf{g}_k\\
         s.t. &\bX = \bB \times_{1} \bPhi \times_2 \bPhi \times_3 \bG\\
         &\bPhi^T \bPhi = \mathbf{I}_s\\
         &{\bG}^T \bG = \mathbf{I}_{K},
    \end{aligned}
    \label{eq:FEN model}
\end{equation}
where  $\mathbf{g}_k$ denotes the $k$th column of $\bG$ and $\alpha_k$ is the smoothing constraint coefficient. $\mathbf{H} = \mathbf{D}^{T}\mathbf{D} \in \bbR^{L\times L}$ denotes differential matrix where $\mathbf{D} \in \bbR^{(L-1)\times L}$ is a matrix which only consists of $0,1,-1$ with elements $1$ at positions $(i,i)$, elements $-1$ at positions $(i,i+1)$, $ i = 1,\cdots,L-1$ and elements $0$ at other positions. $\mathcal{P}_{\boldsymbol{\Omega}}(\cdot)$ denotes element-wise product with $\boldsymbol{\Omega}$.  Then we can interpolate and complete the discrete $\hat{\bX}$ in the third dimension as an estimation of $\mathcal{X}$. When $L$ is large, the interpolation and completion is straightforward. 

To further discuss the estimation properties of Equation (\ref{eq:FEN model}), we first make the following two assumptions. 
\begin{assumption}
    Define $r_i^{0}   \triangleq \textbf{rank}(\bX_{(i)})$, and $\br = [s, s, K, N]$,we assume
    \begin{equation*}
        r_i \leq r_i^{0}, \quad i = 1,\cdots, 4,
    \end{equation*}
    \label{ass:r<R}
    where $r^0_i$ the $i$th element of $\br^0$, $r_i$ is the $i$th element of $\br$.
\end{assumption}
We can always select $\br$ that satisfies Assumption \ref{ass:r<R}.

\begin{assumption}
\label{ass:Omega}
    Given $\bX$, for $\forall \Tilde{\bX} \in \cM_\br$, we have
    \begin{equation*}
        ||\mathcal{P}_{\boldsymbol{\Omega}}(\bX -\Tilde{\bX} )||_{F} \in [c||\bX-\Tilde{\bX} ||_{F}, C||\bX-\Tilde{\bX} ||_{F}],
    \end{equation*}
    where $c,C$ are constants which are only related to $\boldsymbol{\Omega}$,$\bX$ and $\br$. $\cM_\br$ is the low-rank space defined in Equation (\ref{eq:rie low rank}).
\end{assumption}
Intuitively, this assumption requires that $\mathcal{P}_{\boldsymbol{\Omega}}$ is sufficiently
sensitive to the perturbation of $\Tilde{\bX}$. A similar assumption has been used in previous studies on tucker decomposition with smoothness \cite{imaizumi2017tensor}.

\begin{theorem}
\label{the:theorem 2}
Under Assumption \ref{ass:r<R} and Assumption \ref{ass:Omega}, when $\alpha_k = 0$ which means there are no smoothing constraints, the solution of Equation (\ref{eq:FEN model}) will have 
\begin{equation}
    \label{eq:theorem2 1}
    ||\hat{\bX} - \bX||_F \leq \frac{2}{c}||\mathcal{P}_{\boldsymbol{\Omega}}(\bE) ||_F + \frac{C}{c}\sqrt{\sum_{i=1}^3 \delta_i(1-\frac{r_i}{r_i^0})}||\bX ||_F,
\end{equation}
where $\delta_i = 4,\ for\ i = 1,2$ and $\delta_i = 1,\ for \ i = 3$. $C,c$ are the constants in Assumption \ref{ass:Omega}.
\end{theorem}

\begin{proof}
    The proof of Theorem \ref{the:theorem 2} is shown in the Appendix.
\end{proof}

Theorem \ref{the:theorem 2} characterizes an upper bound on the F-norm distance between our estimate $\hat{\bX}$ and the true value $\bX$ when we ignore the smoothing constraints. This upper bound consists of two components: one related to the noise and the other related to the low-rank decomposition. As the hyperparameter $\br$ gets closer to $\mathbf{r}^0$, the upper bound is getting smaller.

\begin{assumption}
\label{ass:unif}
    We define $|\Omega_{ijn}| = \sum_{l=1}^L \Omega_{ijln}$ to denote the number of observable points in function $Y_{ij(t_l)n}, t_l 
    \in \cT$ and $|\Omega_{ijn}^c|= L-|\Omega_{ijn}|$ to denote the number of unobservable points. Then we assume that at least one of the following two scenarios holds,
    \begin{itemize}
        \item (i) $|\Omega_{ijn}| \leq |\Omega_{ijn}^c|$ and $\lceil |\Omega_{ijn}^c|/|\Omega_{ijn}|\rceil = R$ for $\forall i = 1,\cdots,m; j = 1,\cdots,m; n = 1,\cdots,N$.
        \item  (ii) $|\Omega_{ijn}| \geq |\Omega_{ijn}^c|$ and $\lfloor |\Omega_{ijn}|/|\Omega_{ijn}^c|\rfloor = R$ for $\forall i = 1,\cdots,m; j = 1,\cdots,m; n = 1,\cdots,N$,
    \end{itemize}
    where $\lceil \cdot \rceil$ and $\lfloor \cdot \rfloor$  represent round up and round down, respectively.
\end{assumption}
Assumption \ref{ass:unif} indicates that the missing proportions on the functions of different edges in the functional network are approximately uniform. This condition is stronger than the condition with irregular observation but weaker than the condition with regular observation.

\begin{assumption}
\label{ass:upper bound}
    Define $\mathbf{A} = \hat{\bB}_{(3)}(\mathbf{I}_N \otimes \hat{\bPhi} \otimes\hat{\bPhi})^T \in \bbR^{K\times m^2N}$ with respect to $\hat{\bX} = \hat{\bB} \times_1\hat{\bPhi} \times_2 \bPhi \times_3 \hat{\bG}$, i.e., the solution of Equation (\ref{eq:FEN model}). We assume 
    \begin{equation*}
        \sum_{q = 1}^{m^2N} A_{kq}^2 \leq c_k, \ \text{for} \ k = 1,\cdots,K,
    \end{equation*}
    where $c_k$ is the constant which is only related to $\bY$ and $\otimes$ denotes the kronecker product.
\end{assumption}

\begin{theorem}
\label{thm:PomegaC}
    Under Assumption \ref{ass:upper bound}, if the scenario (i) in Assumption \ref{ass:unif} holds, we have
    \begin{equation} 
    \begin{aligned}
    ||\mathcal{P}_{\boldsymbol{\Omega}^C}(\hat{\bX}-\bX) ||_F^2 \leq &3\Delta\Gamma K\sum_{k=1}^K c_k\hat{\mathbf{g}}_k^T\mathbf{H} \hat{\mathbf{g}}_k\\ &+ 3R||\mathcal{P}_{\boldsymbol{\Omega}}(\hat{\bX}-\bY) ||_F^2 + \Tilde{C}.
    \end{aligned}
    \end{equation}
    If the scenario (ii) in Assumption \ref{ass:unif} holds, we have
    \begin{equation} 
    \begin{aligned}
    R||\mathcal{P}_{\boldsymbol{\Omega}^C}(\hat{\bX}-\bX) ||_F^2 \leq &3\Delta'\Gamma' K\sum_{k=1}^K c_k\hat{\mathbf{g}}_k^T\mathbf{H} \hat{\mathbf{g}}_k\\ &+ 3||\mathcal{P}_{\boldsymbol{\Omega}}(\hat{\bX}-\bY) ||_F^2 + \Tilde{C}.
    \end{aligned}
    \end{equation}
    Here $\boldsymbol{\Omega}^C$ is the complementary tensor of $\boldsymbol{\Omega}$, with 0 at positions where $\boldsymbol{\Omega}$ is 1 and 1 at positions where $\boldsymbol{\Omega}$ is 0. $\Delta,\Gamma,\Delta',\Gamma'$ are constants which are only related to $\boldsymbol{\Omega}$, $c_k$ is the constant in Assumption \ref{ass:upper bound} and $\Tilde{C}$ is the constant related to $\bY,\bX$.
\end{theorem}
\begin{proof}
    The proof of Theorem \ref{thm:PomegaC} is shown in the Appendix.
\end{proof}
Theorem \ref{thm:PomegaC} gives the upper bound on the F-norm distance between our estimate
$\hat{\bX}$ and the true value $\bX$ in the unobservable region. It contains the smoothness constraints. In other words, if we set $\alpha_k = \frac{\Delta\Gamma K c_k}{R}$ or $\alpha_k = \Delta'\Gamma' K c_k$, minimizing the objective function in Equation (\ref{eq:FEN model}) will effectively help us minimize the upper bound of $||\mathcal{P}_{\boldsymbol{\Omega}^C}(\hat{\bX}-\bX) ||_F$. This theoretically tells us that adding smoothness constraints can assist us in better estimating the unobservable part.

So far, we have completed the construction of FEN. 
In the next section, we will introduce how to solve the optimization problem of Equation (\ref{eq:FEN model}).

\section{OPTIMIZATION ALGORITHM}
\label{sec:optimization}
\subsection{Riemann Manifold And Optimization}
As the solution space of Equation (\ref{eq:FEN model}) constrains the orthogonality of the basis matrices, the traditional optimization methods may be infeasible under our setting. Therefore, in this section, we introduce how to use Riemann optimization to solve Equation (\ref{eq:FEN model}). 

A Riemann manifold is a smooth surface on which a gradient can be specified at any point in Euclidean space. The low-rank space $\cM_{\br}$ defined by the Tucker decomposition, i.e.,
\begin{equation}
\begin{aligned}
   \cM_\br = \{ &\bX = \bB \times_1 \bPhi \times_2 \bPhi \times_3 \bG | 
    \bB \in \bbR^{s \times s \times K \times N}, \\ &\bPhi \in \bbR^{m \times s}, \bG \in \bbR^{L \times K},
   \bPhi^T \bPhi = I_s, \bG^T \bG = I_K\},
    \label{eq:rie low rank}
\end{aligned}
\end{equation}
is a continuous smooth surface with a gradient vector corresponding to the tangent plane (we will reproduce later). So it is a Riemann manifold. Since we have the target decomposition $\hat{\bX} \in \mathcal{M}_\mathbf{r}$, we need to search the optimal solution $\hat{\bX}$ in the low-rank space $\cM_{\br}$. Then we can reformulate Equation (\ref{eq:FEN model}) as following,
\begin{equation}
\begin{aligned}
   \hat{\bX} =\mathop{\arg \min}\limits_{\bX \in \cM_{\br}} f(\bX)  = &\mathop{\arg \min}\limits_{\bX \in \cM_{\br}} \frac{1}{2}||\mathcal{P}_{\boldsymbol{\Omega}}(\bY - \bX)||_{F}^2 \\ &+ 
         \frac{1}{2}\sum_{k = 1}^K\alpha_k \mathbf{g}_k^{T} \mathbf{H} \mathbf{g}_k,
    \label{eq:FEN model opti}
\end{aligned}
\end{equation}
where $\mathbf{g}_k$ and $\mathbf{H}$ have the same definition with Equation (\ref{eq:FEN model}).

In particular, we choose the conjugate gradient method under Riemann optimization to solve the objective problem. Compared to other methods, such as Newton's method, the conjugate gradient method eliminates the need to solve the Hessian matrix, which saves a significant amount of time, especially in the case of a large number of tensor operations. Compared with the steepest descent method, the conjugate gradient method can converge faster by taking into account the iteration direction of the previous step.
\subsection{Conjugate Gradient Method}
Under the framework of Riemann optimization, the optimization function is restricted to the Riemann manifold, which requires that the optimization method be carried out in the tangent plane of every point on the manifold. So first we need to define a concrete representation of the tangent plane for each point on the Riemann manifold. Then we map the gradient of the objective function in Equation (\ref{eq:FEN model opti}), which is in Euclidean space, to the tangent plane. We call the gradient tensor in the tangent plane as the Riemann gradient. For the conjugate gradient method, we also need to define a gradient transfer projection to update the Riemann gradient in the tangent plane of the last iteration to the tangent plane of the current iteration. Then we linearly combine these two gradients by calculating the linear combination coefficients to get the updated direction in the tangent plane. Last, we need a retraction method to map the updated direction in the tangent plan back to the Riemann manifold. The overall pseudo algorithm is shown in Algorithm \ref{alg:riemann_CG}. It is to be noted in this section we denote $f(\bX_k)$ as the objective function in Equation (\ref{eq:FEN model opti}) where we replace $\bX$ by $\bX_k$ to represent it in the $k$th iteration. 
The details of the conjugate gradient method in Riemann optimization will be introduced as follows. 

\begin{algorithm}
	\caption{Conjugate gradient method under Riemann optimization} 
	\label{alg:riemann_CG}
	\begin{algorithmic}
		\REQUIRE observe tensor $\bY$\\
		mask tensor $\boldsymbol{\Omega}$\\
		rank of Tucker decomposition $\br = [s,s,K, N]$\\
		smoothing constraint coefficient $\alpha_k, k = 1,\cdots, K$\\
		differential matrix $\mathbf{H}$\\
		tolerate error $\delta$\\
		initial value $\bX_0 \in \cM_\br$
		
		\ENSURE the estimated $\{\hat{\bX},\hat{\bB},\hat{\bPhi},\hat{\bG}\}$ of Equation (\ref{eq:FEN model}) 
		
        \STATE $\bm{\eta}_0 = -\mathrm{grad}f(\bX_0)$ 
        \STATE $\gamma_0 = \arg \min_{\gamma}f(\bX_0 + \gamma \bm{\eta}_0) $
        \STATE $\bX_1 = R(\bX_{0}, \gamma_0 \bm{\eta}_0)$
        \STATE $k = 1$
		\WHILE{not converge } 
			\STATE $\bm{\xi}_k = \mathrm{grad}f(\bX_k)$
			\STATE $\bm{\eta}_k =  - \bm{\xi}_k + \beta_k \cF_{\bX_{k-1} \rightarrow \bX_{k}}\bm{\eta}_{k-1}$
			\STATE $\gamma_k = \arg \min_{\gamma}f(\bX_k + \gamma \bm{\eta}_k)$
			\STATE \{$\bX_{k+1},\bB_{k+1},\bPhi_{k+1},\bG_{k+1}  \}= R(\bX_{k}, \gamma_k \bm{\eta}_k)$
			\STATE $k = k+1$
		\ENDWHILE
		\STATE $\{\bX_{out},\bB_{out},\bPhi_{out},\bG_{out}\} = \{\bX_{k},\bB_{k},\bPhi_{k},\bG_{k}\}$
	\end{algorithmic}
\end{algorithm}

\subsubsection{Tangent Plane $T_{\bX}\cM_\br$}
Since the Riemann manifold is a surface in Euclidean space, each step of our iterative process needs to be calculated on the tangent plane of the current point. Specifically, the Riemann gradient $\mathrm{grad}f(\bX_k)$ (the blue vector in Figure \ref{fig:riemann grad}) is the projection of Euclidean gradient $\nabla f(\bX_k)$ (the red vector in Figure \ref{fig:riemann grad}) of the objective function on the tangent plane. As each point on the Riemann manifold has the decomposition shown in Equation (\ref{eq:rie low rank}), we can give the definition of the tangent plane $T_{\bX}\cM_\br$ of a point $\bX = \mathbf{B}\times_1 \bPhi \times_2 \bPhi \times_3 \bG$ as follows:

\begin{equation*}
\begin{aligned}
    T_{\bX}\cM_\br = 
    &\{\Tilde{\mathbf{B}} \times_1 \bPhi \times_2 \bPhi \times_3 \bG + \bB \times_1 \Tilde{\bPhi} \times_2 \bPhi \times_3 \bG + \\
    &\bB \times_1 \bPhi \times_2 \Tilde{\bPhi} \times_3 \bG + 
    \bB \times_1 \bPhi \times_2 \bPhi \times_3 \Tilde{\bG}  | \\
    &\Tilde{\bPhi}^T \bPhi =0, \Tilde{\bG}^T \bG = 0\}.
\end{aligned}
\end{equation*}

\subsubsection{Euclidean Gradient $\nabla f(\bX_k)$}
The Euclidean gradient of the objective function is derived as follows
\begin{equation}
    \begin{aligned}
    \nabla f(\bX_k) &=\mathcal{P}_{\boldsymbol{\Omega}}(\mathbf{Y}-\bX_k)+\text { fold }_{(3)}\left\{\mathbf{H} \bG_k \boldsymbol{\alpha} \mathbf{S}^{T}_k\right\},
    \label{eq:euc grad}
    \end{aligned}
\end{equation}
where $\mathbf{S}_k =\mathbf{Q}_k^{T}\left(\mathbf{Q}_k \mathbf{Q}_k^T\right)^{-1}$ , $\mathbf{Q}_k = \left(\mathbf{B}_k \times_{1} \bPhi_k \times_{2} \bPhi_k \right)_{(3)}$, $\bB_k,\bPhi_k,\bG_k$ are the decomposition of $\bX_k$ on the Riemann manifold, i.e., $\bX_k = \bB_k \times_1 \bPhi_k \times_2 \bPhi_k \times_3 \bG_k$,
$\text{fold}_{(i)}\{ \mathbf{A} \}$ denotes the inverse operation of the matrixization of the tensor $\mathbf{A}$ in the direction $i$, $\boldsymbol{\alpha}$ is the diagonal matrix whose $i$th diagonal element represents the smoothing constraint coefficient $\alpha_i$. We can find that the two terms on the right side of Equation (\ref{eq:euc grad}) correspond to the Euclidean gradient of the two parts of the objective function in Equation (\ref{eq:FEN model opti}) respectively.

\subsubsection{Riemann Gradient $\mathrm{grad}f(\bX_k)$}
Having obtained the tangent plane of the Riemann manifold, we need to project the Euclidean gradient of the objective function $\nabla f(\bX_k)$ onto the tangent plane
and get the Riemann gradient $\mathrm{grad}f(\bX_{k})$. Then we need to define how to project the tensor in Euclidean space onto a given tangent plane.
\begin{equation}
    \begin{aligned}
    \label{eq:euc to rie}
        &\mathbf{P}_{T_{\mathbf{X}} \mathcal{M}_{\mathbf{r}}}(\mathbf{A}): \mathbb{R}^{m \times m \times L \times N}  \rightarrow T_{\mathbf{X}} \mathcal{M}_{\mathbf{r}} \\
        &\mathbf{A} \mapsto  \left(\mathbf{A} \times_1 \bPhi^T \times_2\bPhi^T \times_3 \bG^T\right)  \times_1 \bPhi \times_2\bPhi \times_3 \bG   \\
        &+\bB \times_1 \left( \mathrm{P}_{\bPhi}^{\perp}\left[\mathbf{A} \times_2 \bPhi^T \times_3 \bG^T\right]_{(1)} \mathbf{B}_{(1)}^{\dagger}     \right) \times_2 \bPhi \times_3 \bG  \\
        &+\bB \times_1 \bPhi \times_2 \left( \mathrm{P}_{\bPhi}^{\perp}\left[\mathbf{A} \times_1 \bPhi^T \times_3 \bG^T\right]_{(2)} \mathbf{B}_{(2)}^{\dagger}     \right)  \times_3 \bG \\
        &+\bB \times_1 \bPhi \times_2 \bPhi \times_3\left( \mathrm{P}_{\bG}^{\perp}\left[\mathbf{A} \times_1 \bPhi^T \times_2 \bPhi^T\right]_{(3)} \mathbf{B}_{(3)}^{\dagger}   \right).
    \end{aligned}
\end{equation}
Here $\mathrm{P}_{\bPhi}^{\perp}:=\mathbf{I}_{m}-\bPhi \bPhi^T$, $\mathrm{P}_{\bG}^{\perp}:=\mathbf{I}_{L}-\bG \bG^T$, $\mathbf{B}_{(j)}^{\dagger}=\mathbf{B}_{(j)}^{T}\left(\mathbf{B}_{(j)} \mathbf{B}_{(j)}^{T}\right)^{-1}$, where $\bPhi, \bG, \bB$ are the decomposition of $\mathbf{X}$ on the Riemann manifold, i.e., $\mathbf{X} = \bB \times_1 \bPhi \times_2 \bPhi \times_3 \bG$.
In this way, we can calculate the Riemann gradient $\mathrm{grad}f(\bX_k)$ by replacing $\mathbf{A}, T_{\mathbf{X}} \mathcal{M}_{\mathbf{r}}$ in Equation (\ref{eq:euc to rie}) with $\nabla f(\bX_k), T_{\mathbf{X}_k} \mathcal{M}_{\mathbf{r}}$ in Equation (\ref{eq:euc grad}).

\subsubsection{Vector Transport $\cF_{\bX_{k-1} \rightarrow \bX_{k}}$}

\begin{figure*}
  \centering
  \subfloat[]{%
        \includegraphics[width = 0.3\linewidth]{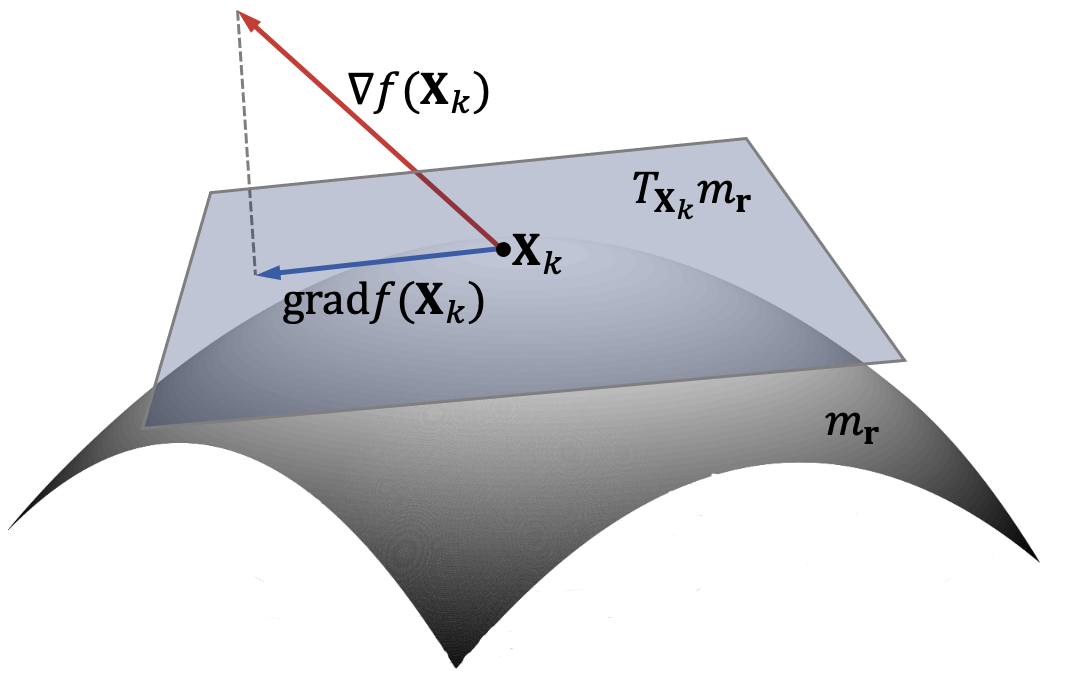}
        \label{fig:riemann grad}}
  \subfloat[]
    { \includegraphics[width = 0.3\linewidth]{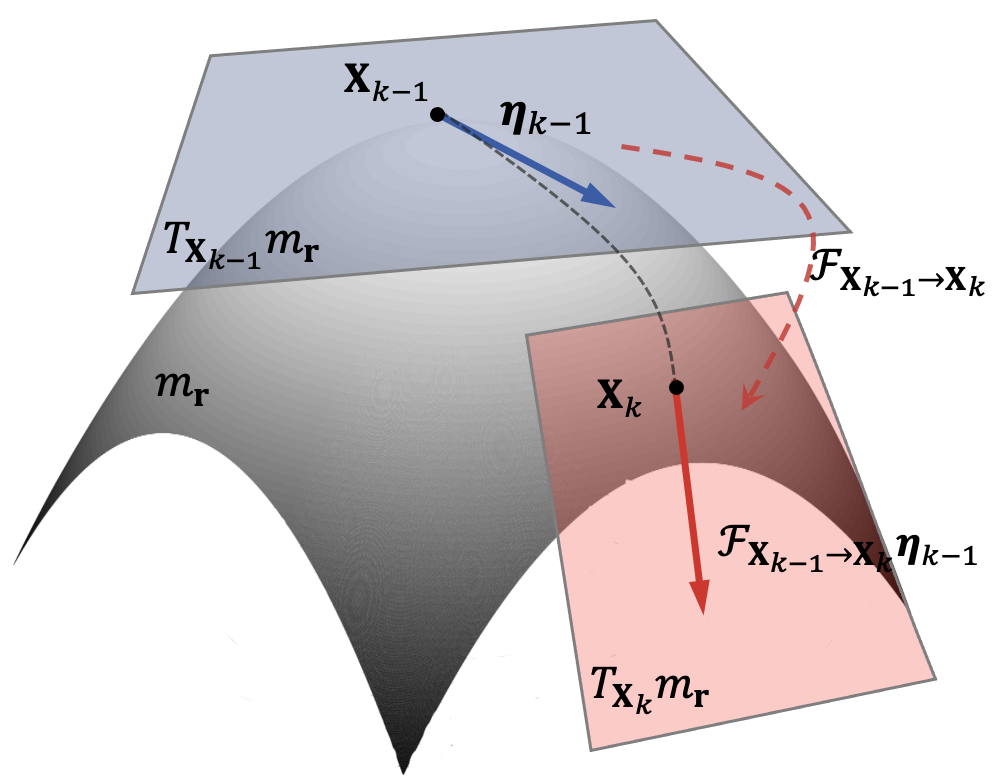}\label{fig:vector tran}}
   \subfloat[]
    {  \includegraphics[width = 0.3\linewidth]{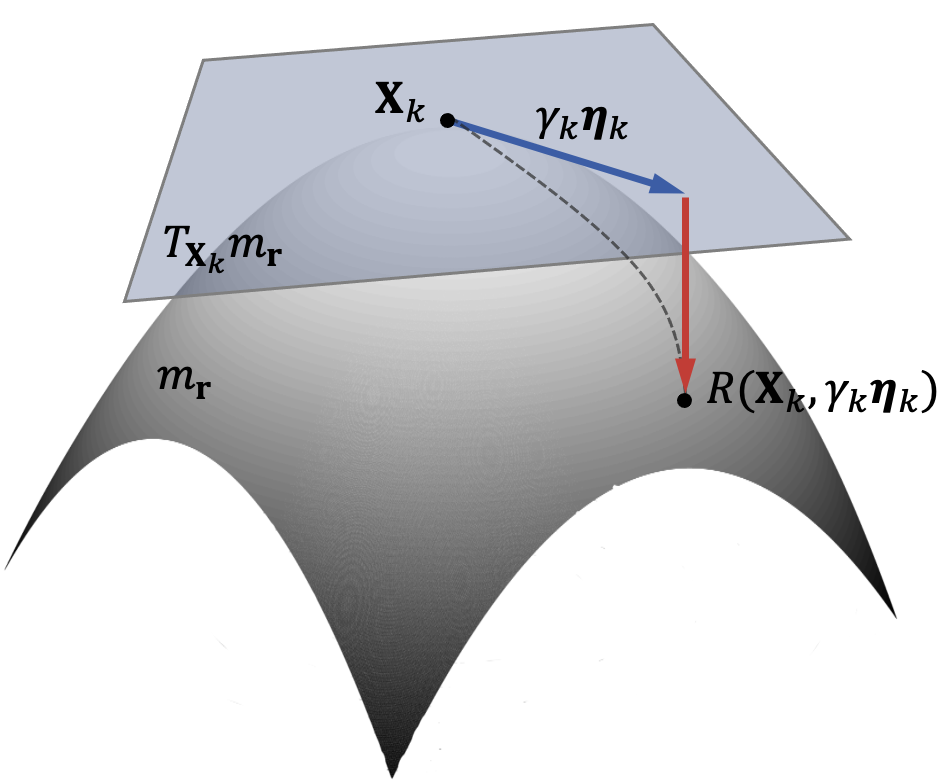}\label{fig:retraction}}
  \caption{Different projections in a Riemannian manifold: (a) Riemann gradient; (b) Vector transport; (c) Retraction}
  \label{fig:vector tran and retraction}
\end{figure*}
In the conjugate gradient method under Riemann optimization, we need to update the current descent direction with considering the direction of the previous iteration.  However as shown in Figure \ref{fig:vector tran}, the descent direction is on the tangent plane of the current point. In order to consider the descent direction of the previous iteration, we need to project it on the tangent plane of the current point, i.e., define the vector transport $\cF_{\bX_{k-1} \rightarrow \bX_{k}}$.

It is worth noting that the tangent plane $T_{\bX}\cM_{\br}$ in the previous iteration is a subset of Euclidean space, so the direction of previous iteration can also be regarded as a tensor in Euclidean space which can be updated in the current tangent plane by Equation (\ref{eq:euc to rie}) as follows
\begin{equation}
    \cF_{\mathbf{X}_{k-1} \rightarrow \mathbf{X}_{k}} \bm{\eta}_{k-1}=\mathbf{P}_{T_{\mathbf{X}_{k}} \mathcal{M}_{\mathbf{r}}}\left(\bm{\eta}_{k-1}\right)
    \label{eq:vector trans}.
\end{equation}
Here $\bm{\eta}_{k-1}$ is the descent direction of the last iteration which is defined on the last tangent plane. By using Equation (\ref{eq:vector trans}), we can get the $\cF_{\mathbf{X}_{k-1} \rightarrow \mathbf{X}_{k}} \bm{\eta}_{k-1}$ that is the projection of the last decent direction onto the current tangent plane. 

\subsubsection{Conjugate Direction $\boldsymbol{\eta}_k$ And Step Size $\gamma_k$}
With the vector transport, we can linearly combine the previous descent direction with the current Riemann gradient to get the current descent direction, i.e.,
\begin{equation*}
    \bm{\eta}_k =  - \mathrm{grad}f(\bX_k) + \beta_k \cF_{\bX_{k-1} \rightarrow \bX_{k}}\bm{\eta}_{k-1},
\end{equation*}
where the linear combination coefficient $\beta_k$ is calculated as following,
\begin{equation*}
    \beta_{\bX_k} = \max\{0, \frac{\langle \mathrm{grad}f(\bX_k), \mathrm{grad}f(\bX_k) - \zeta_{\bX_{k-1} \rightarrow \bX_{k}} \rangle}{||\mathrm{grad}f(\bX_{k-1})||_{F}^2} \},
\end{equation*}
where $\zeta_{\bX_{k-1} \rightarrow \bX_{k}} =  \cF_{\bX_{k-1} \rightarrow \bX_{k}}\mathrm{grad}f(\bX_{k-1})$.
After getting the descent direction, it is necessary to calculate the optimal step size which can be obtained by solving the following optimal problem,
\begin{equation*}
    \gamma_k = \arg \min_{\gamma} f(\bX_k + \gamma \bm{\eta}_k).
\end{equation*}
It has a closed-form solution as 
\begin{equation*}
    \gamma_k = \frac{\langle \mathcal{P}_{\boldsymbol{\Omega}}\bm{\eta}_k, \mathcal{P}_{\boldsymbol{\Omega}}(\bX_k - \bY)\rangle}{\langle \mathcal{P}_{\boldsymbol{\Omega}}\bm{\eta}_k, \mathcal{P}_{\boldsymbol{\Omega}}\bm{\eta}_k \rangle}.
\end{equation*}

Then we can get the value of the next target point on the tangent plane.

\subsubsection{Retraction $R(\cdot)$}
All the current operations are performed on the tangent plane of the Riemann manifold. So at the end of each iteration, the retraction $R(\cdot)$ is to project the points on the tangent plane back onto the Riemann manifold, as shown in Figure \ref{fig:retraction}. 
To satisfy our symmetry constraints for the first two basis matrices, we introduce the Symmetry-HOSVD (SHOSVD) as shown in Algorithm \ref{alg:SHOSVD} where $SVD_{i}(\mathbf{A})$ denotes the first $i$ columns of the left eigenvectors of the SVD decomposition of $\mathbf{A}$. Then we define $R(\bX_k, \gamma_k \bm{\eta}_k) = SHOSVD(\bX_k + \gamma_k \bm{\eta}_k)$. 

\begin{algorithm}[hbt]
	\caption{Sysmmetry-HOSVD} 
	\label{alg:SHOSVD}
	\begin{algorithmic}
		\REQUIRE 
		tensor to be decomposed $\bX \in \bbR^{m \times m \times L \times N}$ \\
        targeted decomposition dimension of the core tensor $\br = [s, s, K, N]$
		\ENSURE 
		decomposition result $SHOSVD(\bX) = \bB \times_1 \bPhi \times_2 \bPhi \times_3 \bG$
		
        \STATE $\bPhi = SVD_{s}\left( \frac{\bX_{(1)} + \bX_{(2)}}{2}\right)$
        \STATE $\bG = SVD_{K}(\bX_{(3)})$
        \STATE $\bB = \bX \times_1 \bPhi^T \times_2 \bPhi^T \times_3 \bG^T$
         \STATE$SHOSVD(\bX) = \bB \times_1 \bPhi \times_2 \bPhi \times_3 \bG$
	\end{algorithmic}
\end{algorithm}


\begin{assumption}
    \label{ass:converage}
    Define $\bX_k$ as the estimation of the $k$th iteration by Algorithm \ref{alg:riemann_CG}, we have 
    \begin{equation*}
        \lim\limits_{k \to \infty} \sigma_{r_i}(\bX_{(i)}) \neq 0, \quad i = 1,2,3,4,
    \end{equation*}
    where $\sigma_j(\bX_{(i)})$ denotes the $j$th biggest singular value of $\bX_{(i)}$.
\end{assumption}
Assumption \ref{ass:converage} states that as $k\to \infty$, $\bX_{(i)}$ has rank $r_i$. This is intuitive because, according to Assumption \ref{ass:r<R}, $\br \leq \br^0$, and $\bY$ also includes the influence of full-rank noise, $\bE$. Therefore, as the algorithm converges, it is impossible to use a lower-rank tensor to estimate $\bY$.

\begin{theorem}
\label{the:local converage}
    Define  $f(\bX_k)$ is the objection function of Equation (\ref{eq:FEN model opti}) by replacing $\bX$ by $\bX_k$. Under Assumption \ref{ass:converage}, we have
    \begin{equation}
        \lim\limits_{k \to \infty} || \mathrm{grad}f(\bX_k)||_F = 0.
    \end{equation}
\end{theorem}
\begin{proof}
    The proof of Theorem \ref{the:local converage} is shown in the Appendix.
\end{proof}
Theorem \ref{the:local converage} shows that Algorithm \ref{alg:riemann_CG} converges at least to a saddle point.

\section{SIMULATION}
\label{sec:simulation}
In the simulation study, we evaluate the performance of FEN and compare it with some state-of-the-art methods. 
\subsection{Introduction Of The Baseline Methods}
As we have introduced in Section \ref{sec:literature}, we consider three types of methods: PCA decomposition, dynamic network modeling, and tensor completion. 

\textbf{PCA-based methods}
\begin{itemize}
    \item VPCA \cite{VPCA} conducted PCA decomposition by concatenating discrete observations of functions from different edges into a longer vector and conducting PCA decomposition. It cannot consider spatial relationships between edges and cannot handle irregularly observed functional data.
    \item MFPCA \cite{MFPCA} conducted PCA decomposition by regarding functions from different edges as the repeated samples of one function and conducting functional PCA. It cannot consider spatial relationships between edges and cannot handle irregularly observed functional data.
    \item SIFPCA \cite{SIFPCA} conducted PCA decomposition by regarding functions from different edges as the repeated samples of one function and conducting functional PCA. It estimated the covariance function by kernel smooth method and conducted the PCA decomposition of the estimated covariance function to do the completion. In this way, it can handle irregularly observed functional data. However, its computational complexity is so high that it is difficult to use in large-scale data.
\end{itemize}

\textbf{Dynamic network modeling-based methods}
\begin{itemize}
    \item MTR \cite{MTR} regarded the observed values of all functions at a time point as a tensor and performed multilinear tensor regression to do the estimation. It considers the spatial relationships between edges, but cannot handle irregularly observed functional data.
    
    \item LDS and MLDS \cite{MLDS}, which are similar to MTR, both treated the functional network as a tensor time series. Yet they use the linear dynamic system or multilinear dynamic system model to do the estimation. So they can not handle irregularly observed functional data as well.
\end{itemize}

\textbf{Tensor completion-based methods}
\begin{itemize}
    \item SPC \cite{CPDecSmooth} treated the functional edges as discrete observations and then conducted CP decomposition with smoothing constraints for adjacency tensor completion and modeling. It does not consider the network community structure. It can handle irregular observation functional data by treating unobserved irregular points as missing data for completion. 
    \item t-TNN \cite{tTNN} conducted the completion by improving the t-SVD decomposition. It treats the functional edges as discrete observations and changes the order of dimensions to do the t-SVD decomposition. But similarly to SPC, it does not consider the network community structure.
\end{itemize}
We compare the above methods with FEN. For the baselines that cannot handle irregular observations, we first use interpolation to get regular observations, and then compare with FEN. 

\subsection{Data Generation}
We introduce data generation for our simulation study. 
According to Equation (\ref{eq:tucker for X}), we need to first generate the core tensor $\bB$, discrete basis matrix $\bPhi$ and functional basis matrix $\mathcal{G}$. 
\begin{algorithm}
	\caption{Generate Core Tensor} 
	\label{alg:generate B}
	\begin{algorithmic}
		\REQUIRE 
		dimension of tensor core  $\br = [s,s,K]$
		
		\ENSURE 
		tensor core $\mathbf{B}$
		
        \STATE $k = 1$
		\WHILE{} 
		    \STATE generate orthogonal matrix $\mathbf{B}_k \in \bbR ^{s \times s}$
		    \STATE $\mathbf{B}_{[i_3 = k]} = \mathbf{B}_k$
		    \IF{Tensor $\mathbf{B}_{[i_3 = 1:k]}$ is full rank in every direction}
		        \STATE $k = k + 1$
		    \ENDIF
			\IF{$k \geq K$}
		        \STATE break
		    \ENDIF
		\ENDWHILE
	\end{algorithmic}
\end{algorithm}

\begin{itemize}
    \item The formation of the core tensor $\bB$ must satisfy that its matrix form has full rank in all directions, i.e., $\textbf{rank}(\mathbf{B}_{(1)}) = \textbf{rank}(\mathbf{B}_{(2)}) = s, \textbf{rank}(\mathbf{B}_{(3)})= K$, where $s,K$ are the hyperparameters. To meet these conditions, we use Algorithm \ref{alg:generate B} to generate $\bB$.
    \item For functional basis matrix $\bPhi$, we randomly generate orthogonal matrix with order $m$ and take the first $s$ columns of it as $\bPhi$.
    \item For functional basis matrix $\mathcal{G}$, we choose Fourier bases $G_{(t)k} = \mathcal{g}_k(t) = sin(k\pi t), k = 1,\cdots, K$ defined on $[-1, 1]$.
\end{itemize}

After getting functional tensor $\mathcal{X}$, we can set the $L$ observation points as equally spaced on $[T_s, T_e]$ to get $\bX$. Next, we add observation noise to $\bX$ by generating a random noise tensor $\mathbf{E}$ from normal distribution with mean 0 and variance $\sigma^2$.

To generate irregular observations, we randomly select a percentage $\omega$ of observations in the fully observed tensor $\bY^F$ to be missing and set the corresponding entries in the binary mask tensor $\boldsymbol{\Omega}$ to 0 while setting the remaining entries to 1. This yields a set of functional network data with irregularly missing observations for our simulation experiments. Figure \ref{fig:miss} shows the generated observation data with different missing percentages.

\begin{figure}
  \centering
  \subfloat[]
    {\includegraphics[width=0.45\linewidth]{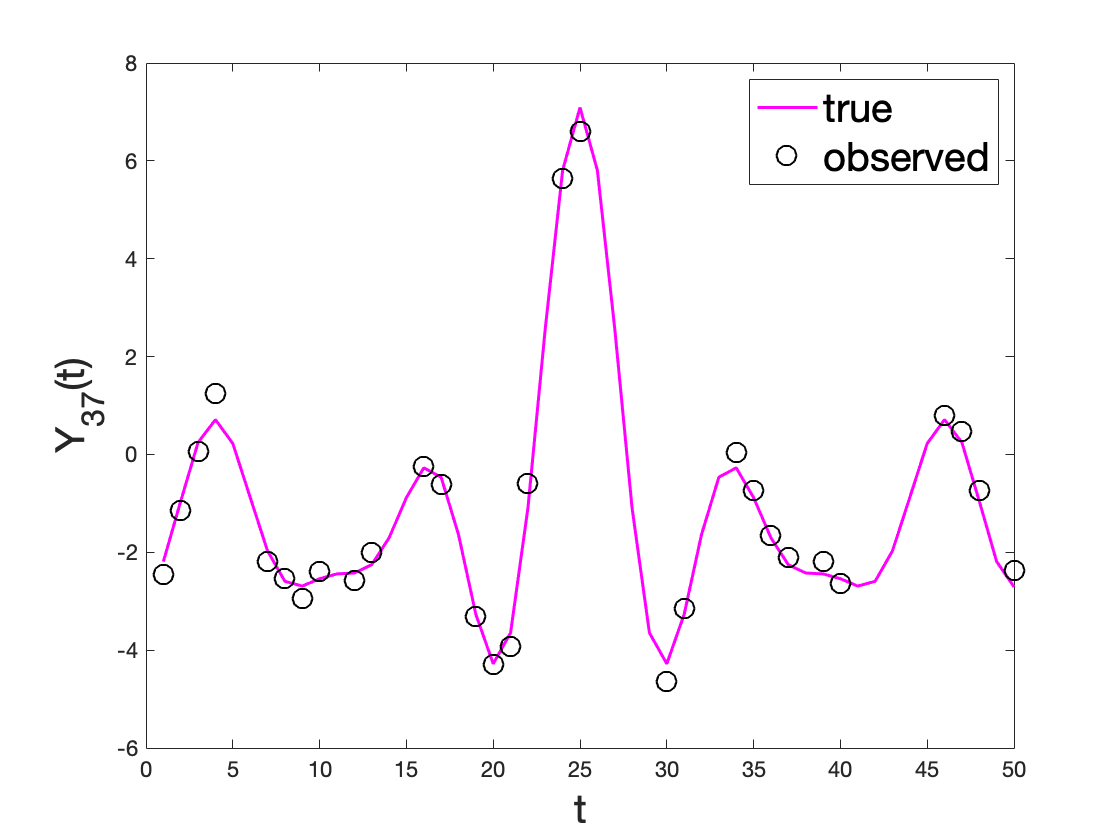}\label{fig:miss-a}}
  \subfloat[]
    {\includegraphics[width=0.45\linewidth]{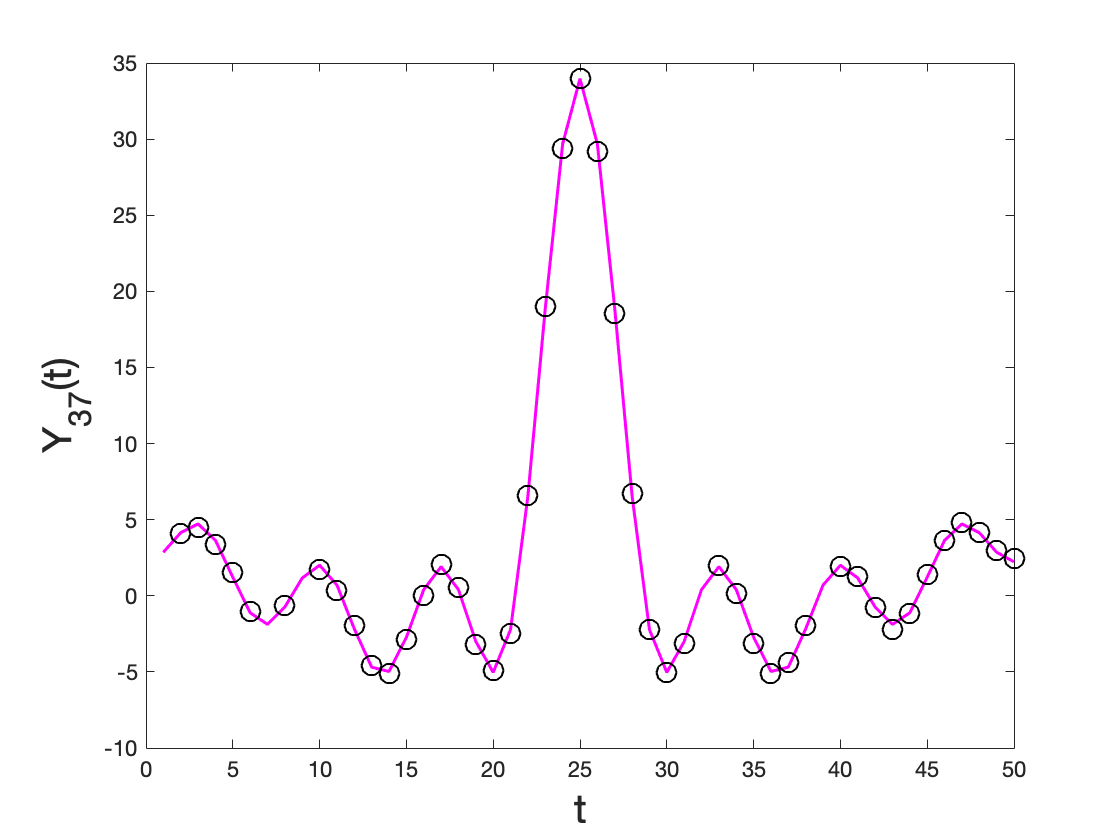}\label{fig:miss-b}}
  \caption{Comparison of observation data generated by different missing percentages: (a)$\omega = 40\%$; (b) $\omega = 10\%$}
  \label{fig:miss}
\end{figure}

\subsection{Small-Scale Simulation}
\label{sec:smallscale}

\begin{table}
	\centering
	\caption{Small-scale simulation parameter setting }
	\resizebox{\linewidth}{!}{
	\begin{tabular}{lccccccc}
		\hline
		\textbf{Parameter} & $dim(\mathcal{X})$ & $L$ & $dim(\bB)$ & $[s,s,K]$   \\
		\textbf{Value} & [10, 10, [-1,1]]& 50 & [3,3,8]& [3, 3, 8]     \\
		\hline 
            \hline
        \textbf{Parameter}&$\sigma^2$ & $\alpha_k$ & $\omega$ \\
        \textbf{Value} & 0.01,0.1,0.2 & 0.1 & 40\% , 30\% , 20\% , 10\% \\
        \hline
	
\end{tabular}}
\label{tab:simulation small 1}
\end{table}

The parameter setting is shown in the Table \ref{tab:simulation small 1}. For each method and each setting of parameters, we generate data to run the method and calculate the sum of squares of fitting errors $SE =||\hat{\bX} - \bX ||_{F}^2$. We repeat the experiments for $N=20$ replications. Table \ref{tab:simout small 1} -Table \ref{tab:simout small 3} show the mean of $SE$, i.e., $MSE$ for each method under different $\sigma^2$ where the standard deviation of $MSE$ is shown in parenthesis.
\begin{table}
    \centering
    \caption{$MSE(\times 10^{-1})$ (with its standard deviation in parenthesis)of different algorithms when $\sigma^2 = 0.01$ at small-scale simulation}
    \resizebox{\linewidth}{!}{
    \begin{tabular}{ccccc}
    \hline
        \textbf{Algorithm} & $\omega  = 40\%$ & $\omega = 30\%$ & $\omega  = 20\%$ & $\omega  = 10\%$ \\ \hline
        \textbf{VPCA} & 73.0(129) & 11.9(5.65) & 3.85(2.65) & 0.611(0.133)  \\ 
        \textbf{MFPCA} & 50.7(38.4) & 15.0(7.75) & 6.83(17.4) & 0.833(0.832)  \\ 
        \textbf{SIFPCA} & 442(48.5) & 441(46.9) & 439(47.5) & 438(47.5)  \\ 
        \textbf{SPC} & 0.032(0.019) & 0.018(0.009) & 0.010(0.006) & 0.004(0.002)  \\ 
        \textbf{t-TNN} & 4.45(5.38) & 0.620(1.14) & 0.067(0.252) & 0.001(0.002)  \\ 
        \textbf{LDS} & 684(423) & 494(118) & 516(130) & 516(116)  \\ 
        \textbf{MLDS} & 271(500) & 36.2(49.1) & 7.12(5.11) & 4.19(7.19)  \\ 
        \textbf{MTR} & 221(130) & 178(128) & 131(93.5) & 128(102)  \\ 
        \textbf{FEN} & \textbf{0.001(0.0001)} & \textbf{0.001(0.00002)} & \textbf{0.001(0.00003)} & \textbf{0.0009(0.00002)} \\  \hline
    \end{tabular}
    }
    \label{tab:simout small 1}
\end{table}

\begin{table}
    \centering
    \caption{$MSE(\times 10^{-1})$ (with its standard deviation in parenthesis) of different algorithms when $\sigma^2 = 0.1$ at small-scale simulation}
    \resizebox{\linewidth}{!}{
    \begin{tabular}{ccccc}
    \hline
        \textbf{Algorithm} & $\omega  = 40\%$ & $\omega  = 30\%$ & $\omega  = 20\%$ & $\omega  = 10\%$ \\ \hline
        \textbf{VPCA} & 35.4(17.0) & 13.8(9.01) & 4.13(2.33) & 0.642(0.278)  \\ 
        \textbf{MFPCA} & 39.2(16.3) & 22.7(47.6) & 5.81(12.3) & 0.581(0.213)  \\ 
        \textbf{SIFPCA} & 443(48.5) & 441(46.9) & 439(47.5) & 439(47.5)  \\ 
        \textbf{SPC} & 0.186(0.139) & 0.101(3.16e-2) & \textbf{0.058(0.023)} & \textbf{0.024(0.023)}  \\ 
        \textbf{t-TNN} & 8.81(11.8) & 1.86(2.80) & 0.247(0.346) & 0.041(0.019)  \\ 
        \textbf{LDS} & 217e1(719e1) & 550(133) & 573(145) & 503(116)  \\ 
        \textbf{MLDS} & 984(303e1) & 30.8(28.1) & 6.55(4.95) & 3.53(7.52)  \\ 
        \textbf{MTR} & 229(135) & 172(128) & 179(143) & 154(123)  \\ 
        \textbf{FEN} & \textbf{0.099(0.002)} & \textbf{0.095(0.002)} & 0.094(0.001) & 0.092(0.002) \\ \hline
    \end{tabular}
   }
    \label{tab:simout small 2}
\end{table}

\begin{table}
    \centering
    \caption{$MSE(\times 10^{-1})$ (with its standard deviation in parenthesis) of different algorithms when $\sigma^2 = 0.2$ at small-scale simulation}
    \resizebox{\linewidth}{!}{
    \begin{tabular}{ccccc}
    \hline
        \textbf{Algorithm} & $\omega  = 40\%$ & $\omega  = 30\%$ & $\omega = 20\%$ & $\omega  = 10\%$ \\ \hline
        \textbf{VPCA} & 196(574) & 10.4(3.31) & 5.04(3.30) & 0.623(0.158)  \\ 
        \textbf{MFPCA} & 58.7(49.5) & 14.5(8.95) & 3.08(1.05) & 0.834(0.336)  \\ 
        \textbf{SIFPCA} & 442(48.6) & 442(47.0) & 439(47.5) & 439(47.5)  \\ 
        \textbf{SPC} & 1.81(0.466) & 0.903(0.262) & 0.408(0.087) & \textbf{0.142(0.025)}  \\ 
        \textbf{t-TNN} & 6.13(6.90) & 1.49(0.638) & 0.714(0.899) & 0.141(0.036)  \\ 
        \textbf{LDS} & 704(681) & 534(128) & 522(149) & 525(114)  \\ 
        \textbf{MLDS} & 256(655) & 17.4(8.69) & 10.3(9.80) & 2.67(5.00)  \\ 
        \textbf{MTR} & 233(127) & 177(136) & 124(122) & 124(122)  \\ 
        \textbf{FEN} & \textbf{0.397(0.011)} & \textbf{0.381(0.010)} & \textbf{0.375(0.007)} & 0.369(0.007) \\ \hline
    \end{tabular}
   }
    \label{tab:simout small 3}
\end{table}

We observe that the $MSE$ of each method decreases as the variance of observation noise $\sigma^2$ or the missing percentage $\omega$ decreases, which is in line with our expectations. FEN outperforms the other baseline methods and is less sensitive to the missing percentage. Since VPCA, MFPCA, and SIFPCA treat different functions as repeated samples of the same function, which loses the topological structure of the graph data, they have quite poor performance. As to LDS, MLDS, and MTR, they cannot handle irregular observations. Consequently, they have to first interpolate the data as a preprocessing step, and then be applied for analysis, which leads to suboptimal results. Among them, MLDS performs the best but is still strongly affected by the missing percentage. The other two tensor completion algorithms, SPC and t-TNN, have a bit worse performance than FEN. Since they do not consider the graph community structure, their interpretability is not as good as FEN.

\subsection{Large-Scale Simulation}
\label{sec:bigscale}
Following almost the same setting as Section \ref{sec:smallscale}, Table \ref{tab:simulation big 1} shows the parameters for our large-scale simulation.

\begin{table}[bht]
	\centering
	\caption{Large-scale simulation parameter setting}
	\resizebox{\linewidth}{!}{
	\begin{tabular}{lccccccc}
		\hline
		\textbf{Parameter} & $dim(\mathcal{X})$ & $L$ & $dim(\bB)$ & $[s,s,K]$ \\
		\textbf{Value} & [50, 50, [-1,1]]& 100 & [15,15,25]& [15, 15, 25] &   \\
		\hline
            \hline
            \textbf{Parameter}&$\sigma^2$ & $\alpha_k$ & $\omega $ \\
            \textbf{Value} & 0.01,0.1,0.2 & 0.1 & 40\% , 30\% , 20\% , 10\% \\
            \hline
\end{tabular}}
\label{tab:simulation big 1}
\end{table}
For comparison, we also calculate the $MSE$ and its standard deviation based on 20 experiment replications. The simulation results are shown in Table \ref{tab:simout big 1} -Table \ref{tab:simout big 3}.

\begin{table}
    \centering
    \caption{$MSE(\times 10)$ (with its standard deviation in parenthesis) of different algorithms when $\sigma^2 = 0.01$ at large-scale simulation}
    \resizebox{\linewidth}{!}{
    \begin{tabular}{ccccc}
    \hline
        \textbf{Algorithm} & $\omega  = 40\%$ & $\omega  = 30\%$ & $\omega  = 20\%$ & $\omega  = 10\%$ \\ \hline
        \textbf{VPCA}& 218(277) & 79.3(22.9) & 15.0(1.87) & 4.03(0.349)  \\ 
        \textbf{MFPCA} & 178(70.2) & 46.3(5.38) & 16.2(2.82) & 4.02(0.446)  \\ 
        \textbf{SPC} & 0.516(0.089) & 0.298(0.013) & 0.158(0.005) & 0.068(0.004)  \\ 
        \textbf{t-TNN} & 1.30(0.378) & 0.008(0.009) & 0.0006(0.0002) & 0.0002(0.00004)  \\ 
        \textbf{LDS} & 86.4(25.1) & 65.2(1.32) & 58.4(0.729) & 55.4(0.599)  \\ 
        \textbf{MLDS} & 127(73.1) & 71.7(19.4) & 56.6(6.29) & 53.1(0.551)  \\ 
        \textbf{MTR} & 144(25.6) & 97.0(11.7) & 73.2(3.16) & 62.2(1.88)  \\ 
        \textbf{FEN} & \textbf{0.0(0.0)} & \textbf{0.0(0.0)} & \textbf{0.0(0.0)} & \textbf{0.0(0.0)} \\ \hline
    \end{tabular}
   }
     \label{tab:simout big 1}
\end{table}

\begin{table}
    \centering
    \caption{$MSE(\times 10)$ (with its standard deviation in parenthesis) of different algorithms when $\sigma^2 = 0.1$ at large-scale simulation}
    \resizebox{\linewidth}{!}{
    \begin{tabular}{ccccc}
    \hline
        \textbf{Algorithm} & $\omega  = 40\%$ & $\omega  = 30\%$ & $\omega  = 20\%$ & $\omega = 10\%$ \\ \hline
        \textbf{VPCA} & 165(92.8) & 64.2(71.9) & 15.7(2.14) & 4.12(0.342)  \\ 
        \textbf{MFPCA} & 189(131) & 50.2(14.7) & 15.2(1.63) & 3.98(0.278)  \\ 
        \textbf{SPC} & 0.506(0.015) & 0.295(0.010) & 0.157(0.006) & 0.690(0.0260)  \\ 
        \textbf{t-TNN} & 1.23(0.269) & 0.005(0.002) & 0.001(0.0002) & 0.004(0.0)  \\ 
        \textbf{LDS} & 86.5(25.7) & 65.2(1.32) & 58.4(0.728) & 55.4(0.599)  \\ 
        \textbf{MLDS} & 127(73.2) & 72.0(19.3) & 56.7(6.50) & 53.1(0.552)  \\ 
        \textbf{MTR} & 153(48.4) & 96.8(15.6) & 73.7(2.22) & 62.3(1.38)  \\ 
        \textbf{FEN} & \textbf{0.0010(0.0)} & \textbf{0.0010(0.0)} & \textbf{0.0010(0.0)} & \textbf{0.0010(0.0)} \\ \hline
    \end{tabular}
    }
    \label{tab:simout big 2}
\end{table}

\begin{table}
    \centering
    \caption{$MSE(\times 10)$ (with its standard deviation in parenthesis) of different algorithms when $\sigma^2 = 0.2$ at large-scale simulation}
    \resizebox{\linewidth}{!}{
    \begin{tabular}{ccccc}
    \hline
        \textbf{Algorithm} & $\omega  = 40\%$ & $\omega  = 30\%$ & $\omega = 20\%$ & $\omega  = 10\%$ \\ \hline
        \textbf{VPCA} & 173(96.4) & 46.6(13.8) & 17.0(6.06) & 4.14(0.404)  \\ 
        \textbf{MFPCA} & 234(320) & 64.1(43.2) & 15.1(2.92) & 4.03(0.234)  \\ 
        \textbf{SPC} & 2.88(0.072) & 1.64(0.047) & 0.853(0.024) & 0.334(0.009)  \\ 
        \textbf{t-TNN} & 1.25(0.285) & 0.030(0.019) & 0.005(0.0005) & \textbf{0.002(0.00004)}  \\ 
        \textbf{LDS} & 86.6(26.5) & 65.2(1.32) & 58.4(0.728) & 55.4(0.599)  \\ 
        \textbf{MLDS} & 127(73.2) & 72.2(19.2) & 56.8(6.70) & 53.1(0.552)  \\ 
        \textbf{MTR} & 127(72.4) & 101(9.03) & 73.6(4.86) & 62.5(2.36)  \\ 
        \textbf{FEN} & \textbf{0.004(0.00001)} & \textbf{0.004(0.0)} & \textbf{0.004(0.00001)} & 0.004(0.00001) \\ \hline
    \end{tabular}
    }
    \label{tab:simout big 3}
\end{table}
The large-scale simulation experiments further demonstrate the superiority of FEN over other algorithms, especially as the network dimension and number of observation points increase. Furthermore, the model is robust to different percentages of missing data. For the other baselines, their performances are also similar as those in the small-scale simulations.  


\section{CASE STUDY}
\label{sec:case}
In this section, we consider two datasets in urban metro transportation to evaluate the performance of FEN and its baselines. As we know, understanding passenger movement patterns, and exploring the spatial relationships of metro stations, can help better public transportation management and land use planning. We can formulate the metro network as a graph, by treating each station as a node, and real-time passenger flows between different Origin-Destination (O-D) paths as edges. As Figure \ref{fig:example6} shows, each passenger flow curve has a very smooth profile and can be regarded as functional data. Consequently, all the O-D path flows formulate a functional-edged graph. However, the observation points of the passenger flows for different O-D paths are different, resulting in irregularly observed functions. By analyzing such data, we hope to identify the urban centers, hubs, and socioeconomic clusters based on network centralities and community structures.

\subsection{Hong Kong Metro System Data}

\begin{figure*}
  \centering
  \includegraphics[width = 0.75\linewidth]{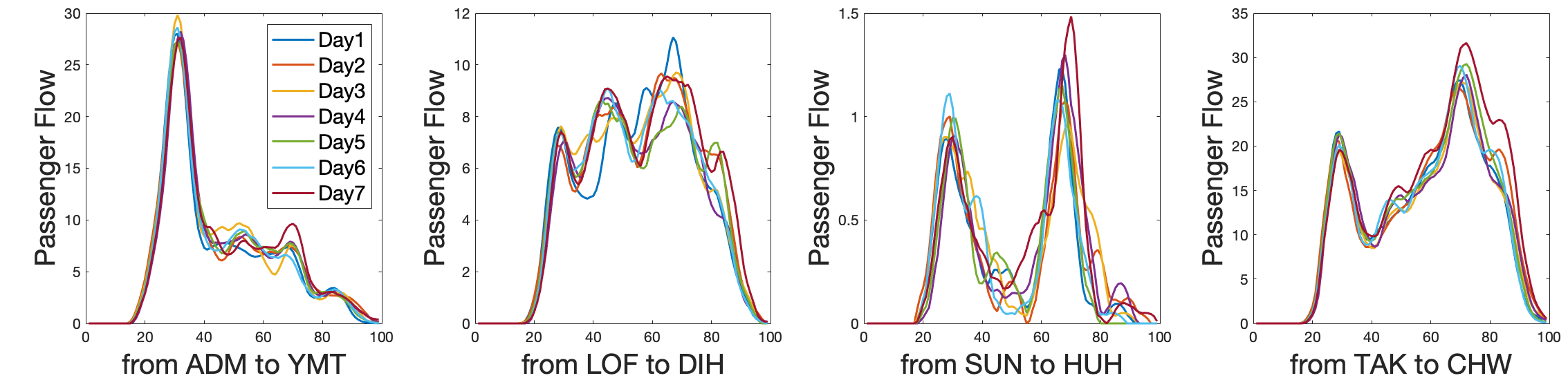}
  \caption{Passenger flow functions between some O-D paths of the Hong Kong metro system}
  \label{fig:example6}
\end{figure*}
The Hong Kong metro system has in total 90 stations, i.e., nodes. The passenger flows between any two stations from 5:00 am to 12:00 pm every day are treated as functional edges. We have in total 24 days' data for analysis. It is noted that if two stations have accumulated passenger flows of less than 20 each day on average, we remove the edge between these two stations. The parameter settings of FEN for the dataset are shown in
Table \ref{tab:HK para}, in which the hyperparameter $\br = [s,s,K,N]$ is the optimal hyperparameter selected by pre-examination. We set the size of the observation points $L = 99$. The missing data percentage in reality is 10\%. To better evaluate the performance of FEN, we further consider different values of $\omega$ by randomly removing some observations.

\begin{table}
	\centering
	\caption{Parameter setting for the Hong Kong metro system data}
	\resizebox{\linewidth}{!}{
	\begin{tabular}{lccccc}
		\hline
		\textbf{Parameter} & $dim(\mathcal{X})$ & $L$ &$[s,s,K,N]$   \\
		\textbf{Value} & [90, 90, [5:00 am, 12:00 pm], 24]& 99  & [23,23,20,24]  \\
		\hline
            \hline
            \textbf{Parameter} & $\omega $ &  $\alpha_k$ \\
            \textbf{Value}    & 40\%, 30\%, 20\%, 10\%&  0.1 \\
            \hline
\end{tabular}}
\label{tab:HK para}
\end{table}

\begin{table*}[bth]
    \centering
    \caption{$MSE$ or $MSE_{test}$ of different algorithms for the Hong Kong metro system data}
    \scalebox{0.8}{
    \begin{tabular}{ccccccc}
    \hline

         & \multicolumn{3}{c}{\textbf{Tensor completion-based methods}} & \multicolumn{3}{c}{\textbf{PCA-based methods}} \\
         & \multicolumn{3}{c}{(with standard deviation of $MSE$ in parenthesis)} &  \multicolumn{3}{c}{(with $MSE_{train}$ in parenthesis)} \\
         
         \cmidrule(lr){2-4}\cmidrule(lr){5-7}
        $\omega $ & \textbf{SPC} & \textbf{t-TNN} & \textbf{FEN} &  \textbf{VPCA} & \textbf{MFPCA} & \textbf{FEN} \\ 
        \hline
        $40\%$ &5.30(0.283) & 11.3(0.358) & \textbf{4.12(0.247)} &  23.0(22.7) & \textbf{12.0(2.63)} & 13.0(4.20)  \\ 
        $30\%$& 4.07(0.225) & 7.83(0.345) & \textbf{4.03(0.236)} &  22.8(22.5) & 11.8(2.62) & \textbf{10.9(5.73)}  \\ 
        $20\%$& 3.79(0.197) & 3.81(0.331) & \textbf{3.61(0.211)} &  22.8(22.5) & 11.7(2.62) & \textbf{8.99(6.57)}  \\ 
        $10\%$& 3.61(0.189) & 3.77(0.214) & \textbf{3.59(0.179)} &  22.8(22.5) & 11.7(2.62) & \textbf{7.79(7.59)} \\ \hline
    \end{tabular}}
    \label{tab:Case HK}
\end{table*}

To better illustrate the fitting performance of FEN, we selected five stations which including three central subway stations and two relatively remote stations. We draw their fitted paired passenger flows under the missing percentage of $\omega = 40\%$ to show that FEN has a very satisfactory fitting performance. The selected stations and fitting results are both shown in the appendix. 

We further compare FEN with other baselines, tensor completion-based methods and PCA-based methods, in two different ways. Due to the ultra-high computational cost, dynamic network modeling-based methods do not participate in the comparison. The results are shown in Table \ref{tab:Case HK}. As the tensor completion methods can not handle multiple samples, we set the sample number $N = 1$ of FEN model and make the comparison. We use the data for 24 days to repeat these algorithms and calculate $SE_{miss}  = ||\mathcal{P}_{\boldsymbol{\Omega}^{C}}(\hat{\bX} - \bX) ||_{F}^2$ for each day. Then we calculate their mean, i.e., $MSE$ and its standard deviation, as shown in the first three columns of Table \ref{tab:Case HK}.
As to the PCA-based algorithms, they can handle multiple samples. So we estimate their model parameters as well as FEN by multiple samples. In particular, we devide the 24 days' data into a training set (20 days) and a test set (4 days). We can calculate the $MSE_{train}$ of PCA-based algorithms and FEN, and control the $MSE_{train}$ of all the methods similar and as small as far as possible, by selecting the proper number of principal components in PCA-based algorithms and $K$ in FEN. Then we fix the estimated PCA loadings and only update the PCA scores for the testing data. For FEN, similarly, we fix $\bPhi,\bG$ as those estimated by the training data, and only update $\mathbf{B}$ for the testing data, i.e., 
\begin{equation}
    \begin{aligned}
    \label{eq:test core}
    \hat{\bB}_{test} &= \bY_{test} \times_1 \hat{\bPhi}^T \times_2 \hat{\bPhi}^T \times_3 \hat{\bG}^T \\
    \hat{\bX}_{test} &= \hat{\bB}_{test} \times_1 \hat{\bPhi} \times_2 \hat{\bPhi} \times_3 \hat{\bG}.
    \end{aligned}
\end{equation}
Then we can calculate $MSE_{test}$ and $MSE_{train}$ of all the methods, as shown in the last columns of Table \ref{tab:Case HK}.

Comparing FEN with tensor completion-based methods, we find that FEN generally outperforms SPC and t-TNN, although the difference is not as big as that in the simulation studies.
Comparing FEN to PCA-based methods, we observe that FEN generally achieves lower $MSE_{test}$ values. We have to mention that VPCA exhibits very poor performance, and its $MSE_{train}$ and $MSE_{test}$ values remain high even though all its principal components are used. Overall, these results suggest that FEN offers advantages over other algorithms. 

\subsection{Singapore Metro System Data }
Compared to the Hong Kong metro system above, the Singapore metro system has in total 153 stations, i.e., nodes, which places higher demands on the adaptability and efficiency of the model. We have in total 16 days' data for analysis.  It is noted that if two stations have accumulated passenger flows of less than 20 each day on average, we remove the edge between these two stations. The parameter settings of FEN for the dataset are shown in Table \ref{tab:SG para}, in which the hyperparameter $\br = [s,s,K,N]$ is the optimal hyperparameter selected by pre-examination.
Similarly, we set the size of the observation points $L = 99$. The missing data percentage in reality is 10\%. To better evaluate the performance of FEN, we further consider different values of $\omega$ by randomly removing some observations.
 
\begin{table}[bht]
	\centering
	\caption{Parameter setting for the Singapore metro system data}
	\resizebox{\linewidth}{!}{
	\begin{tabular}{lcccccc}
		\hline
		\textbf{Parameter} & $dim(\mathcal{X})$ & $L$  & $[s,s,K,N]$  \\
		\textbf{Value} & [153, 153, [5:00 am, 12:00 pm], 16]& 99 & [23,23,25,16]   \\
		\hline
            \hline
            \textbf{Parameter} & $\omega $ & $\alpha_k$\\
            \textbf{Value}  & 40\%, 30\%, 20\%, 10\%  &  0.1 \\
            \hline
  
\end{tabular}}
\label{tab:SG para}
\end{table}

The data processing and comparison method of the baseline algorithms is similar to the Hong Kong metro system and for the PCA-based methods, we use 12 days' data as the train set and 4 days' data as test set. The results are shown in Table \ref{tab:Case Singapore}.

\begin{table*}[tbh]
    \centering
    \caption{$MSE$ or $MSE_{test}$ of different algorithms for the Singapore metro system data}
    \scalebox{0.8}{
    \begin{tabular}{ccccccc}
    \hline
        & \multicolumn{3}{c}{\textbf{Tensor completion-based methods}} &  \multicolumn{3}{c}{\textbf{PCA-based methods}} \\
        & \multicolumn{3}{c}{(with standard deviation of $MSE$ in parenthesis)} &  \multicolumn{3}{c}{(with $MSE_{train}$ in parenthesis)} \\
         \cmidrule(lr){2-4}\cmidrule(lr){5-7}
        $\omega $ & \textbf{SPC} & \textbf{t-TNN} & \textbf{FEN} &  \textbf{VPCA} & \textbf{MFPCA} & \textbf{FEN} \\ \hline
        $40\%$ & 0.601(0.243) & 1.35(0.544) & \textbf{0.265(0.143)} &  1.15(1.15) & \textbf{0.394(0.180)} & 0.476(0.206)  \\ 
        $30\%$&0.572(0.173) & 1.15(0.347) & \textbf{0.259(0.215)} &  1.14(1.14) & 0.393(0.180) & \textbf{0.387(0.217)}  \\ 
        $20\%$&0.571(0.115) & 1.05(0.203) & \textbf{0.268(0.230)} &  1.14(1.14) & 0.392(0.179) & \textbf{0.324(0.233)}  \\ 
        $10\%$&0.528(0.0534) & 0.558(0.0564) & \textbf{0.285(0.253)} &  1.14(1.14) & 0.388(0.178) & \textbf{0.274(0.255)} \\ \hline
    \end{tabular}}
    \label{tab:Case Singapore}
\end{table*}

When comparing FEN to tensor completion-based methods, it is evident that FEN tends to outperform SPC and t-TNN, although the difference is not as pronounced as observed in the simulation studies and the Hong Kong metro system studies.
In contrast, when comparing FEN to PCA-based methods, we consistently observe that FEN achieves lower values of $MSE_{test}$. Notably, VPCA demonstrates poor performance, as its $MSE_{train}$ and $MSE_{test}$ values remain high even when all its principal components are used. In summary, these results collectively indicate that FEN presents clear advantages over other algorithms.

In this section, we compare FEN model with other baseline methods using Hong Kong and Singapore metro system data. We find that FEN model has advantages over the baseline methods in modeling functional network data with irregular observations. 

\section{CONCLUSION}
\label{sec:conlusion}


In this paper, we introduced a novel method, the FEN model, tailored for modeling functional network data. Our model effectively addresses key challenges encountered in functional network analysis, including function smoothing, dimension reduction, and handling irregular observations. To estimate the FEN model, we developed a Riemann optimization algorithm founded on the conjugate gradient method. Furthermore, we provided theoretical guarantees for our algorithm, including its convergence to a local optimal solution and the upper bound between the optimal solution and the true value.

Our comprehensive evaluation, encompassing simulation studies and case studies involving real data from Hong Kong and Singapore subways, underscored the superiority of the FEN model over several baseline models for completing functional network data. Nonetheless, our research has also unveiled areas for future exploration. For instance, we observed that accurately estimating the core tensor of the test set by Equation (\ref{eq:test core}), a critical element for precise performance evaluation, becomes challenging when the proportion of missing observations is substantial. Therefore, a promising avenue for future research lies in developing algorithms dedicated to the accurate estimation of the core tensor for test set data.

\begin{small}
	\bibliography{sample_paper.bbl}
        \bibliographystyle{IEEEtran}
\end{small}

\begin{IEEEbiography}[{\includegraphics[width=1.00in,height=1.25in,clip,keepaspectratio]{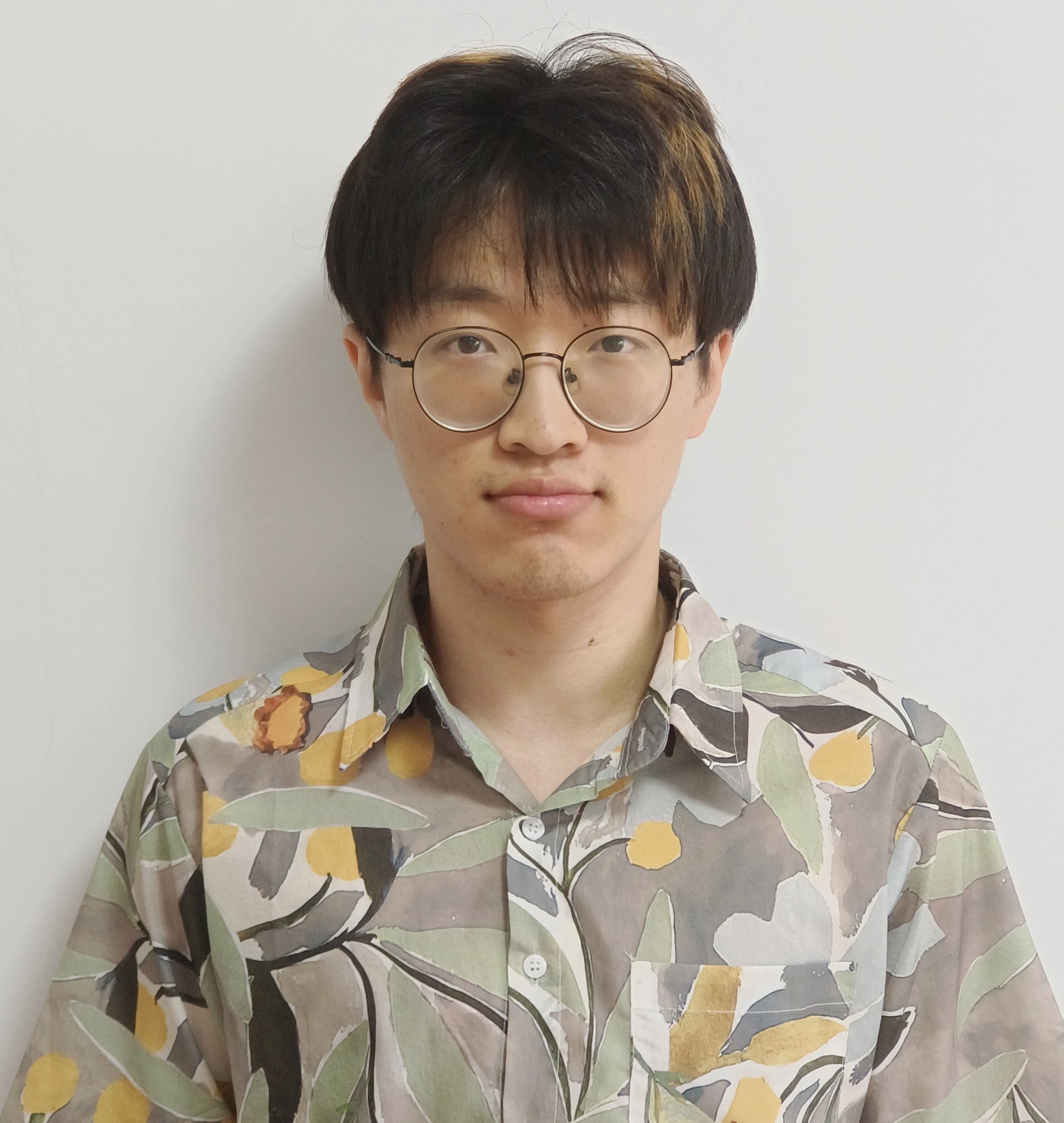}}]{Haijie Xu} received the B.S. degree from the Department of  Industrial Engineering, Tsinghua University, Beijing, China, in 2022. 
He is currently a Ph.D. candidate with the Department of Industrial Engineering, Tsinghua University, Beijing, China. His research interests include sequential change detection, functional data analysis.
\end{IEEEbiography}

\begin{IEEEbiography}[{\includegraphics[width=1.00in,height=1.25in,clip,keepaspectratio]{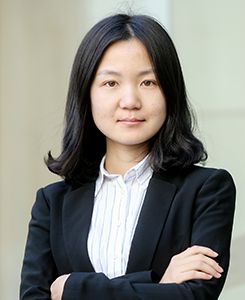}}]{Chen Zhang} received the B.S. degree in electronic science and technology (optics) from Tianjin University, Tianjin, China, in 2012, and the Ph.D. degree in industrial systems engineering and management from the National University of Singapore, Singapore, in 2017. 
From 2017 to 2018, she was a Research Fellow with School of Information Systems, Singapore Management University, Singapore. She is currently an Associate Professor at the Department of Industrial Engineering, Tsinghua University, Beijing, China. Her research interests include statistics and machine learning with applications in industrial data analysis and medical data analysis. 
\end{IEEEbiography}

\end{document}